\documentclass[twoside]{article} 
\usepackage[accepted]{aistats2021}
\usepackage{hyperref}
\usepackage{amsmath, amssymb, amsthm,amsbsy, amscd, bm}
\usepackage{graphicx}
\graphicspath{{figs/}}
\usepackage{subfigure}
\usepackage{algorithm}
\usepackage{algorithmicx}
\usepackage{booktabs}
\usepackage[noend]{algpseudocode}
\usepackage{footnote}

\numberwithin{equation}{section}
\DeclareMathOperator*{\plim}{plim}

\newcommand{\E}{\mathbb{E}}
\DeclareMathOperator{\prox}{\textup{prox}}

\DeclareMathOperator{\sgn}{sign}

\newcommand{\R}{\mathbb{R}}

\newcommand{\abs}[1]{\lvert#1\rvert}

\renewcommand{\mathbf}{\bm}

\newcommand{\bet}{\boldsymbol{\beta}}
\newcommand{\bethat}{\bm{\hat\beta}}
\newcommand{\p}{p}
\newcommand{\Eta}{\bm\eta}
\newcommand{\n}{n}

\newcommand{\B}{\mathbf{B}}
\newcommand{\Z}{\mathbf{Z}}
\newcommand{\X}{\mathbf{X}}
\newcommand{\w}{\mathbf{w}}
\newcommand{\y}{\mathbf{y}}

\newcommand{\blam}{\boldsymbol \lambda}
\newcommand{\thet}{\boldsymbol \theta}
\newcommand{\z}{\mathbf{z}}

\newcommand{\bfalph}{\boldsymbol \alpha}

\newcommand{\MSE}{\texttt{MSE}}

\newtheorem{theorem}{Theorem}
\newtheorem{othertheorem}{othertheorem}[section]

\newtheorem{fact}[othertheorem]{Fact}

\begin{document}

\twocolumn[
\aistatstitle{Efficient Designs of SLOPE Penalty Sequences in Finite Dimension}

\aistatsauthor{Yiliang Zhang$^*$ \And Zhiqi Bu$^*$}

\aistatsaddress{ University of Pennsylvania} ]

\begin{abstract}
In linear regression, SLOPE is a new convex analysis method that generalizes the Lasso via the sorted $\ell_1$ penalty: larger fitted coefficients are penalized more heavily. This magnitude-dependent regularization requires an input of penalty sequence $\blam$, instead of a scalar penalty as in the Lasso case, thus making the design extremely expensive in computation. In this paper, we propose two efficient algorithms to design the possibly high-dimensional SLOPE penalty, in order to minimize the mean squared error. For Gaussian data matrices, we propose a first order Projected Gradient Descent (PGD) under the Approximate Message Passing regime. For general data matrices, we present a zero-th order Coordinate Descent (CD) to design a sub-class of SLOPE, referred to as the $k$-level SLOPE. Our CD allows a useful trade-off between the accuracy and the computation speed. We demonstrate the performance of SLOPE with our designs via extensive experiments on synthetic data and real-world datasets.
\end{abstract}

\section{Introduction}
In sparse linear regression, we aim to find an accurate sparse estimator $\hat{\bm\beta}$ of the unknown truth $\bm\beta$ from
$$\y=\X\bet+\w.$$
Here the response $\y\in\R^\n$, the data matrix $\X\in\R^{\n\times \p}$, the true parameter $\bet\in\R^\p$ and the noise $\w\in\R^\n$. Specifically, in high dimension where $\p>\n$, ordinary linear regression fails to find a unique solution and $L_1$-related regularization is usually introduced to achieve sparse estimators, including the Lasso \cite{tibshirani1996regression}, elastic net \cite{zou2005regularization}, (sparse) group Lasso \cite{yuan2006model}, adaptive Lasso \cite{zou2006adaptive} and the recent SLOPE \cite{bogdan2015slope}:
\begin{align}
  \widehat{\boldsymbol{\beta}}(\blam)=\underset{\boldsymbol{b}}{\arg \min } \frac{1}{2}\|\boldsymbol{y}-\boldsymbol{X} \boldsymbol{b}\|_{2}^{2}+\sum_{i=1}^{p} \lambda_{i}|b|_{(i)} 
\label{eq:SLOPE}
\end{align}
Here $\sum_{i=1}^{p} \lambda_{i}|b|_{(i)}$ is the \textit{sorted $\ell_1$ norm} of $\bm b$ governed by the penalty vector $\blam\in\R^\p$ with $\lambda_{1} \geq \cdots \geq \lambda_{p}\geq 0$, and $\abs b_{(i)}$ is the ordered statistics of absolute values $\abs{b_i}$ such that $|b|_{(1)} \geq \cdots \geq |b|_{(p)}$.
We pause here to remark that the Lasso is a sub-case of SLOPE when $\lambda_1=\cdots=\lambda_p$, since there is no need to sort and hence sorted $\ell_1$ norm is simply the $\ell_1$ norm. Generally speaking, the sorting step in the norm allows SLOPE to work in a way similar to the taxation, assigning larger thresholds to larger fitted coefficients.

Many desirable properties have been proven for SLOPE. For example, SLOPE is a convex optimization that can be solved by existing gradient methods, such as the subgradient descent and the proximal gradient descent; SLOPE achieves minimax estimation properties without requiring knowledge of the sparsity degree of $\bet$ \cite{su2016slope}; SLOPE controls the false discovery rate in the case of independent predictors. However, understanding the SLOPE problem is difficult. Questions such as what posterior distribution does SLOPE solution follow, can we characterize statistics (e.g. the false discovery rate and true positive rate) from SLOPE exactly, whether SLOPE has better estimation error than the Lasso, are not answered until recently \cite{bellec2018slope,bu2019algorithmic,hu2019asymptotics}. Still, the substantial difficulty imposed by the sorted penalty impedes the general application of SLOPE for two reasons. From the practical point of view, tuning a $\R^\p$ penalty can be extremely costly for large $p$ (e.g. in high dimensional regression or over-parameterized neural networks) and naive methods that work for the Lasso, such as the grid search, renders not pragmatic. From a theoretical perspective, the sorted norm is complicated in that the effect of thresholding of SLOPE is  \textit{non-separable} and \textit{data-dependent}, unlike the Lasso, thus making the analysis much involved. 

In this paper, we further exploit the advantage of the data-depending penalty in SLOPE and investigate, from the estimation error perspective, how to design the SLOPE penalty sequence to achieve better performance. 

We give a computationally efficient framework to design the SLOPE penalty sequence $\blam\in\R^\p$ which corresponds to an estimator $\bethat(\blam)$ that minimizes the estimation error. To be more specific, we derive the gradient of penalty for SLOPE under the Approximate Message Passing (AMP) regime \cite{bayati2011dynamics, bayati2011lasso,donoho2010message,donoho2009message} and propose the \textit{$k$-level} SLOPE for the general data matrcies. In words, $k$-level SLOPE is a sub-class of SLOPE, where the $p$ elements in $\{\lambda_i\}$ have only $k$ unique values. Under this definition, the general SLOPE is $p$-level SLOPE and the Lasso is indeed 1-level SLOPE. Additionally, $k$-level SLOPE is a sub-class of $(k+1)$-level SLOPE, and larger $k$ leads to better performance but requires longer computation time. As a result, by choosing an appropriate $k$, we can establish a trade-off between speed and accuracy. We illustrate in various experiments that such a trade-off is of practical use as even a small $k$ may improve the performance non-trivially.

\subsection{Notations}
We start by introducing the proximal operator of SLOPE,
\begin{align}
  \operatorname{prox}_{J_{\boldsymbol{\theta}}}(\boldsymbol{y}):=\underset{\bm b}{\operatorname{arg min}} \frac{1}{2}\|\boldsymbol{y}-\boldsymbol{b}\|^{2}+J_{\boldsymbol{\theta}}(\boldsymbol{b}),
\end{align}
where $J_{\thet}(\bm b):= \sum_{i=1}^p \theta_i|b|_{(i)}$ and the proximal operator indeed solves \eqref{eq:SLOPE} with an identity data matrix. This operator is the building block that is iteratively applied to derive the SLOPE estimator in the proximal gradient descent (ISTA \cite{daubechies2004iterative}) and in FISTA \cite{beck2009fast}. We note that there is no closed form of $\operatorname{prox}_{J_{\boldsymbol{\theta}}}(\boldsymbol{x})$ but it can be efficiently computed as in \cite[Algorithm 3]{bogdan2015slope}. Next we denote the mean squared error (\MSE) between two vectors in $\mathbb{R}^m$ as $\MSE(\bm u,\bm v):=\|\bm u-\bm v\|^2/m$. Two performance measures that are investigated in this work are the \textit{prediction error} with $m=n$, $\MSE(\bm y,\hat{\bm y})$, and the \textit{estimation error} with $m=p$, $\MSE(\bet,\hat\bet)$.

\section{SLOPE penalty design under AMP regime}

\subsection{Computing the gradients with respect to the penalty}
We introduce a special regime of the AMP for SLOPE \cite{bu2019algorithmic}, within which the SLOPE estimator can be asymptotically exactly characterized. A similar regime is the case when Convex Gaussian Min-max Theorem (CGMT) \cite{celentano2020lasso,thrampoulidis2018precise,thrampoulidis2015regularized,thrampoulidis2014gaussian} applies, which shares similar assumptions as those of AMP. We then derive the gradient of $\MSE(\bet,\bethat)$ with respect to the penalty $\blam$ and optimize our penalty design iteratively. Generally speaking, AMP is a class of gradient-based optimization algorithms that mainly work on independent Gaussian random data matrices, offering both a sequence of estimators that converges to the true minimizer and a distributional characterization of the latter (see \cite[Theorem 3]{bu2019algorithmic} and \cite[Theorem 1]{hu2019asymptotics}). Here we present the five assumptions of the SLOPE AMP \cite{bu2019algorithmic}:
\begin{itemize}
    \item The data matrix $\X$ has independent and identically-distributed (i.i.d.) gaussian entries that have mean 0 and variance $1/n$.
    
    \item The signal $\mathbf{\beta}$ has elements that are i.i.d. and follow $\Pi$, with $\mathbb{E}\left(\Pi^{2} \max \{0, \log \Pi\}\right)<\infty$.
    
    \item The noise $\bm w$ is elementwise i.i.d. and follows $W$, with $\sigma_{w}^{2}:=\mathbb{E}\left(W^{2}\right)<\infty$.
    
    \item The vector $\boldsymbol{\lambda}(p)=\left(\lambda_{1}, \ldots, \lambda_{p}\right)$ is elementwise i.i.d. and follows $\Lambda$, with $\mathbb{E}\left(\Lambda^{2}\right)<\infty$.
    
    \item The ratio $n/p$ reaches a constant $\delta \in (0,\infty)$ in the large system limit, as $n$ and $p \rightarrow \infty$.

\end{itemize} 
Under these assumptions, Theorem 3 in \cite{bu2019algorithmic} provides an asymptotic characterization of $\bethat$, which can be informally interpreted as
\begin{align}
\bethat(\bm\lambda)\approx \operatorname{prox}_{J_{\bfalph\tau}}\left(\boldsymbol{\beta}+\tau\boldsymbol{Z}\right)
\label{eq:AMP result}
\end{align}
in which $(\bfalph(\blam),\tau(\blam))$ are the unique solutions of two key equations, namely the finite-dimension approximation of the \textit{calibration} and the \textit{state evolution} in the AMP (or CGMT) regime (see \cite{bu2019algorithmic,hu2019asymptotics}; note that AMP is an asymptotic theory):
\begin{align}
  \boldsymbol{\lambda}&=\boldsymbol{\alpha} \tau\left(1-\frac{1}{n} \mathbb{E}\left\|\operatorname{prox}_{J_{\bfalph\tau}}\left(\bet+\tau \boldsymbol{Z}\right)\right\|_{0}^{*}\right)
  \label{eq:calibration}
  \\
  \tau^{2}&=\sigma_w^{2}+\frac{1}{\delta p} \mathbb{E}\left\|\operatorname{prox}_{J_{\bfalph\tau}}\left(\boldsymbol{\beta}+\tau\boldsymbol{Z}\right)-\boldsymbol{\beta}\right\|^{2}
\label{eq:state evolution}
\end{align}
Here we assume the noise $\w$ has variance $\sigma_w^2$, $\|\cdot\|_0^*$ is a modified $\ell_0$ norm that counts the unique non-zero absolute values in a vector and $\bm Z\in\R^\p$ is a vector in which each element is i.i.d. standard normal. We denote $\delta:=\lim_p \n/\p$ as the aspect ratio or sampling ratio and $\epsilon:=\lim_p |\{j:\beta_j\neq 0\}|/p$. 

Using \eqref{eq:AMP result}, we observe that to minimize $\MSE(\bet,\bethat)$ is equivalent to finding desirable $(\bfalph,\tau)$. We now introduce some properties that are useful in deriving the desirable $\blam$, which uniquely defines $(\bfalph,\tau)$. By \cite[Proposition 2.3]{bu2019algorithmic}, the calibration \eqref{eq:calibration} describes a bijective, monotone and parallel mapping $\Lambda$\footnote{
For a given $\bfalph$, we can use (\ref{eq:state evolution}) to obtain a unique $\tau(\bfalph)$ and leverage (\ref{eq:calibration}) to obtain a corresponding penalty vector $\blam(\bfalph)$.} between $\bfalph$ and $\blam$ \cite[Proposition 2.3]{bu2019algorithmic}, which allows us to work with $\bfalph$ easily instead of $\blam$. By \cite[Theorem 1]{bu2019algorithmic}, the state evolution \eqref{eq:state evolution} can be solved via a fixed point recursion, which converges to the unique solution $\tau(\bfalph)$ monotonically under any initial condition.

Under AMP region, our strategy is to design $\blam\in\R^\p$ in SLOPE that, by quoting \cite[Corollary 3.2]{bu2019algorithmic}, minimizes:
\begin{align*}
\plim\|\hat\bet-\bet\|^2/\p=\delta(\tau^2-\sigma_w^2)
\end{align*}
where $\plim$ is the probability limit. Hence minimizing $\MSE(\bet,\bethat)$ is in fact equivalent to minimizing $\tau$, which depends on $\bfalph$ and leads to differentiating \eqref{eq:state evolution} against each of $\alpha_i$ for $i\in [\p]$. In what follows, we view the scalar $\tau$ as a function of the penalty $\bfalph$ given the prior. Next, we use the gradient information to descend (with the projection elaborated in Algorithm \ref{alg: projection}) till convergence. Once the minimizer $\bfalph$ is obtained, we leverage the calibration \eqref{eq:calibration} to map to the corresponding $\blam(\bfalph)$. 

In what follows, we shorthand $\prox_{J_{\bm b}}(\bm a)$ by using $\Eta(\bm a;\bm b)$. In particular $\prox_{J_{\bfalph\tau}}(\boldsymbol{\beta}+\tau\boldsymbol{Z})$ is denoted by $\Eta$ and we let $\eta_j$ represent its $j$-th element. We define a set $I_j := \{k:|\eta_k|=|\eta_j|\}$, which will be used in characterization of gradients.  We also define an inverse mapping for ranking of indices: $\sigma: \{1,\dots,p\} \to \{1,\dots,p\}$ such that $\sigma(i) = j$ representing $|\Eta|_{(i)} =  |\eta_j|$. Consider a toy example $\Eta = (-2,-4,3,1)$, then the ranking of magnitudes is $(3,1,2,4)$ whose inverse gives: $\sigma(1) = 2$, $\sigma(2) = 3$, $\sigma(3) = 1$ and $\sigma(4) = 4$. This mapping is useful in assigning the penalties to coefficients in $\bethat$ due to the sorting procedure.

We state the following theorem to give a concrete form of gradients $\partial\tau / \partial \alpha_i$, which is used in the projected gradient descent (PGD) in Algorithm \ref{alg: PGD}.
\begin{theorem}
The gradients satisfy
\begin{align}
\frac{\partial\tau}{\partial\alpha_i} =
\E \frac{1}{|I_{\sigma(i)}| D(\mathbf{\alpha},\tau)}\sum_{j \in  I_{\sigma(i)}}(\eta_j-\beta_j)\sgn(\eta_j)\tau
\end{align}
where $D(\mathbf{\alpha},\tau)$ is a negative constant that is independent of index $i$.
\label{thm:gradient}
\end{theorem}


Here the expectation is taken with respect to $\Z$ in $\bm \eta$, which in turn also affects $I_{\sigma(i)}$. The detailed form of $D(\bfalph,\tau)$ in the denominator and the proof of Theorem \ref{thm:gradient} can be found in Appendix \ref{a1}, where we also claim that $D(\mathbf{\alpha},\tau)$ is always negative. In practice, we can either set the step size $s_t$ as constant or simply set $D=-1$ to save computation time. We remark that, using a constant step size $s$ and $-\E\left(\sum_{j \in  I_{\sigma(i)}}(\eta_j-\beta_j)\sgn(\eta_j)\tau\right)/|I_{\sigma(i)}|$ as the gradient is equivalent to using a time-dependent $s_t=s\cdot |D(\bfalph^t,\tau_t)|$ and the actual gradient $\frac{\partial\tau}{\partial\alpha_i}$.

\subsection{Projection onto non-negative decreasing vectors}
We notice that $\blam$ must be non-negative and decreasing, and that $\bfalph$ must be decreasing as it is parallel to $\blam$ by \eqref{eq:calibration}. Hence the vanilla gradient descent is unsuitable for this constrained optimization problem of $\bfalph$. We design a projected gradient descent (PGD) in the following. To do so, we first give Algorithm \ref{alg: projection} to compute the projection and establish the correctness of the algorithm in Theorem \ref{pgd}.

Let $\mathcal{S}$ denote the set of non-negative and decreasing vectors in $\R^\p$ (i.e. $\blam \in\mathcal{S}\Rightarrow\lambda_i \geq \lambda_{i+1}\geq 0, \forall i$). Define the projection on to $\mathcal{S}$ as 
\begin{align}
\Pi_{\mathcal{S}}(\mathbf{\gamma})= \textup{argmin}_{\mathbf{\gamma'} \in \mathcal{S}} \frac{1}{2}\|\mathbf{\gamma} - \mathbf{\gamma'}\|^2_2.
\label{eq:projection}
\end{align}


\begin{algorithm}[H]
    \begin{algorithmic}
    \State\textbf{Input: } Arbitrary sequence $\bm \gamma = (\gamma_1, \dots, \gamma_p)$
	\For{$i=1,\cdots,p$}
    \State $\triangleright$ Identify the shortest sub-sequence $\{\gamma_{j}, \dots, \gamma_i\}$ whose average is smaller than its left neighbor (with $\gamma_0=\infty$):
    $$
    \frac{1}{i-j+1}\sum_{k=j}^i \gamma_k \leq \gamma_{j-1} \ \ 
    $$
    \State $\triangleright$ Assign the average value to such sub-sequence for $(\gamma_{j}, \dots, \gamma_i)$:
    $$
    \gamma_{j}, \dots, \gamma_i \leftarrow \frac{1}{i-j+1}\sum_{k=j}^i \gamma_k
    $$
	\EndFor
	\State\textbf{Output: } $\max\{\mathbf{\gamma},0\}$
	\Comment{Element-wise truncation}\\
	\end{algorithmic}
	\caption{\texttt{ProjectOnS} ($\Pi_\mathcal{S}$)}
	\label{alg: projection}
\end{algorithm}
We show that Algorithm \ref{alg: projection} indeed finds the minimizer of \eqref{eq:projection} and provide the proof in Appendix \ref{a2}.

\begin{theorem} \label{pgd}
Given an arbitrary $\mathbf{\gamma} \in \mathbb{R}^p$ as input, Algorithm \ref{alg: projection} outputs the projection of $\mathbf{\gamma}$ on $\mathcal{S}$, that is, $\Pi_{\mathcal{S}}(\mathbf{\gamma})$.
\end{theorem}

On high level, the proof consists of two parts. In the first part we provide a detailed characterization of $\Pi_\mathcal{S}(\gamma)$ by partitioning the index sequence $\{1,\dots, p\}$ into a number of carefully selected sub-sequences. We prove that within each sub-sequence, $\Pi_{\mathcal{S}}(\mathbf{\gamma})$ takes the same value at each index, and such value is exactly the average of the sub-sequence $\bm \gamma$'s values at these indices. In the second part, we prove that Algorithm 1 indeed finds such sub-sequences and thus operates in a way that matches the goal of the projection $\Pi_\mathcal{S}(\bm\gamma)$. The final truncation of the averaged sequence at 0 is a trivial method to guarantee the non-negativity.

\begin{figure}[!htp]
	\centering
	\includegraphics[width=6cm,height=4cm]{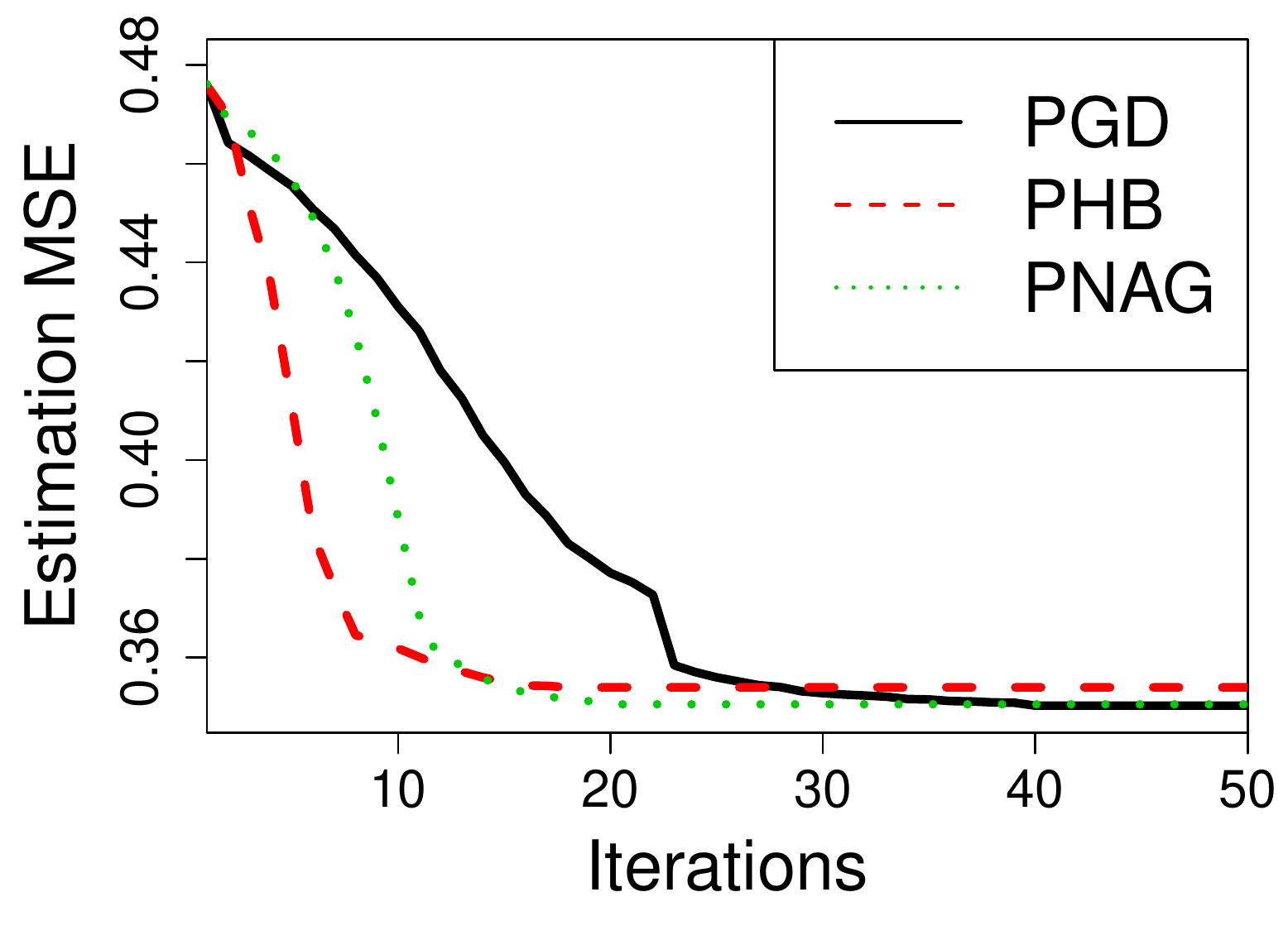}
	\includegraphics[width=5.8cm,height=4cm]{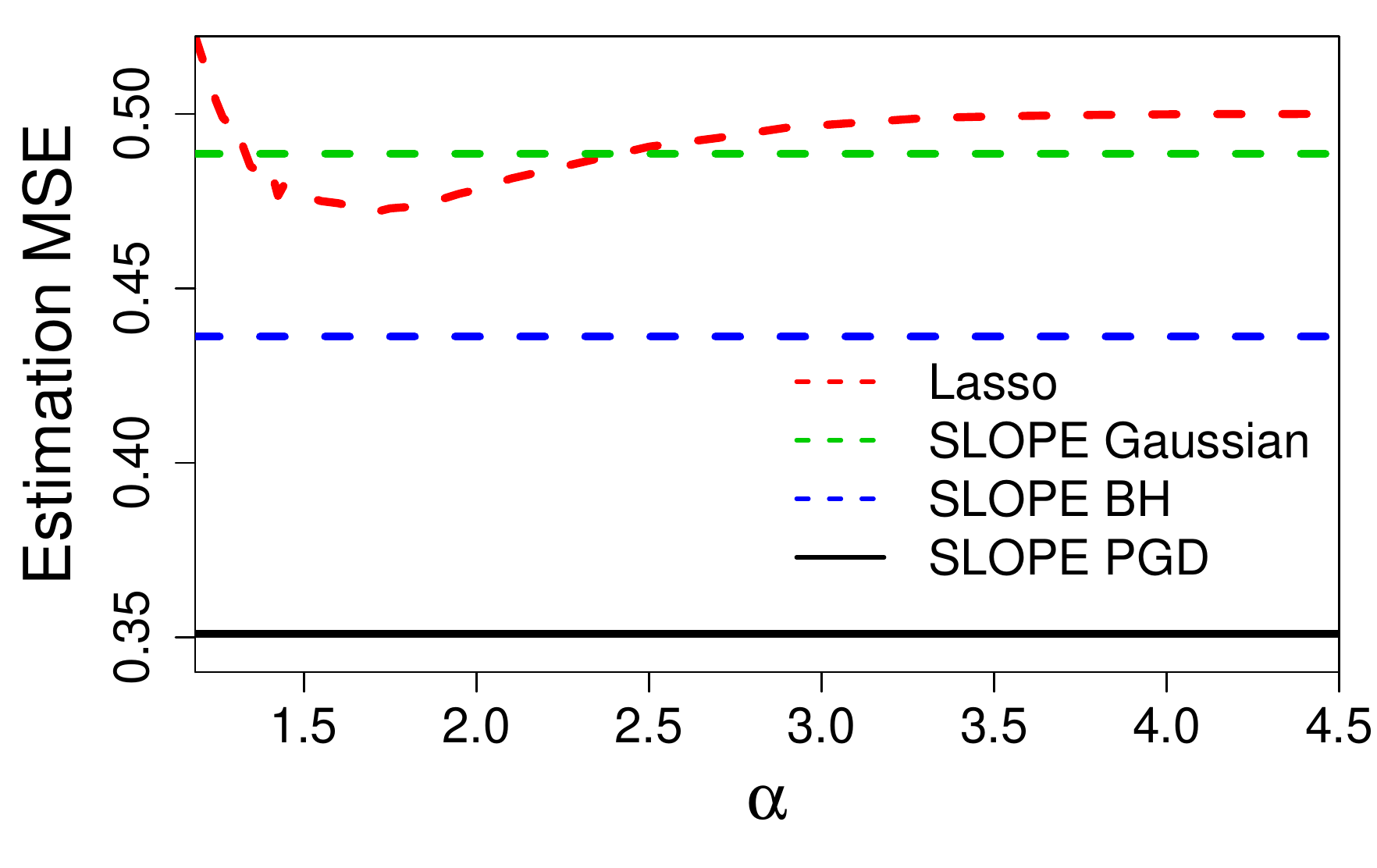}
	\includegraphics[width=6.5cm]{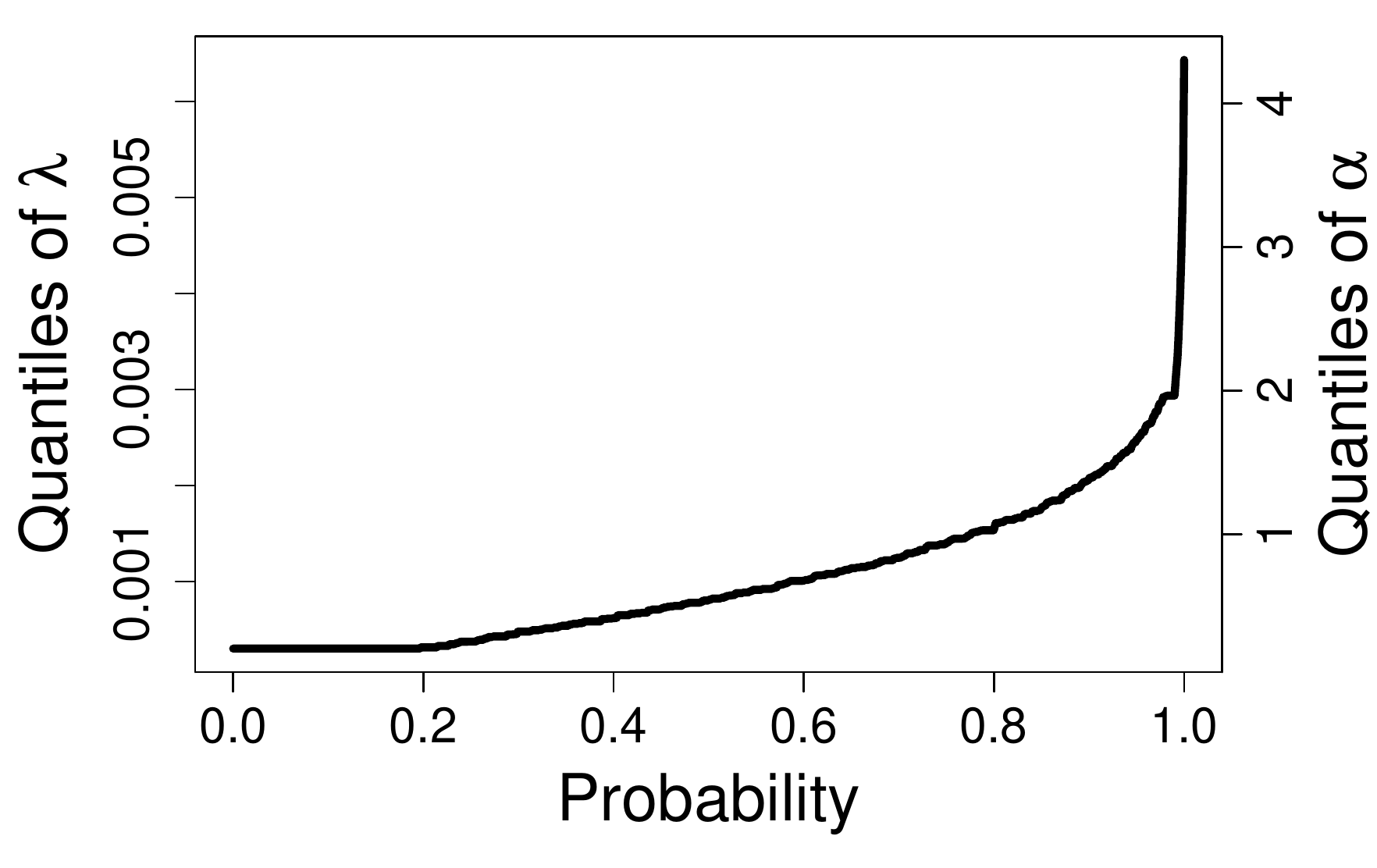}
	\vspace{0.1cm}
	\includegraphics[width=6.5cm]{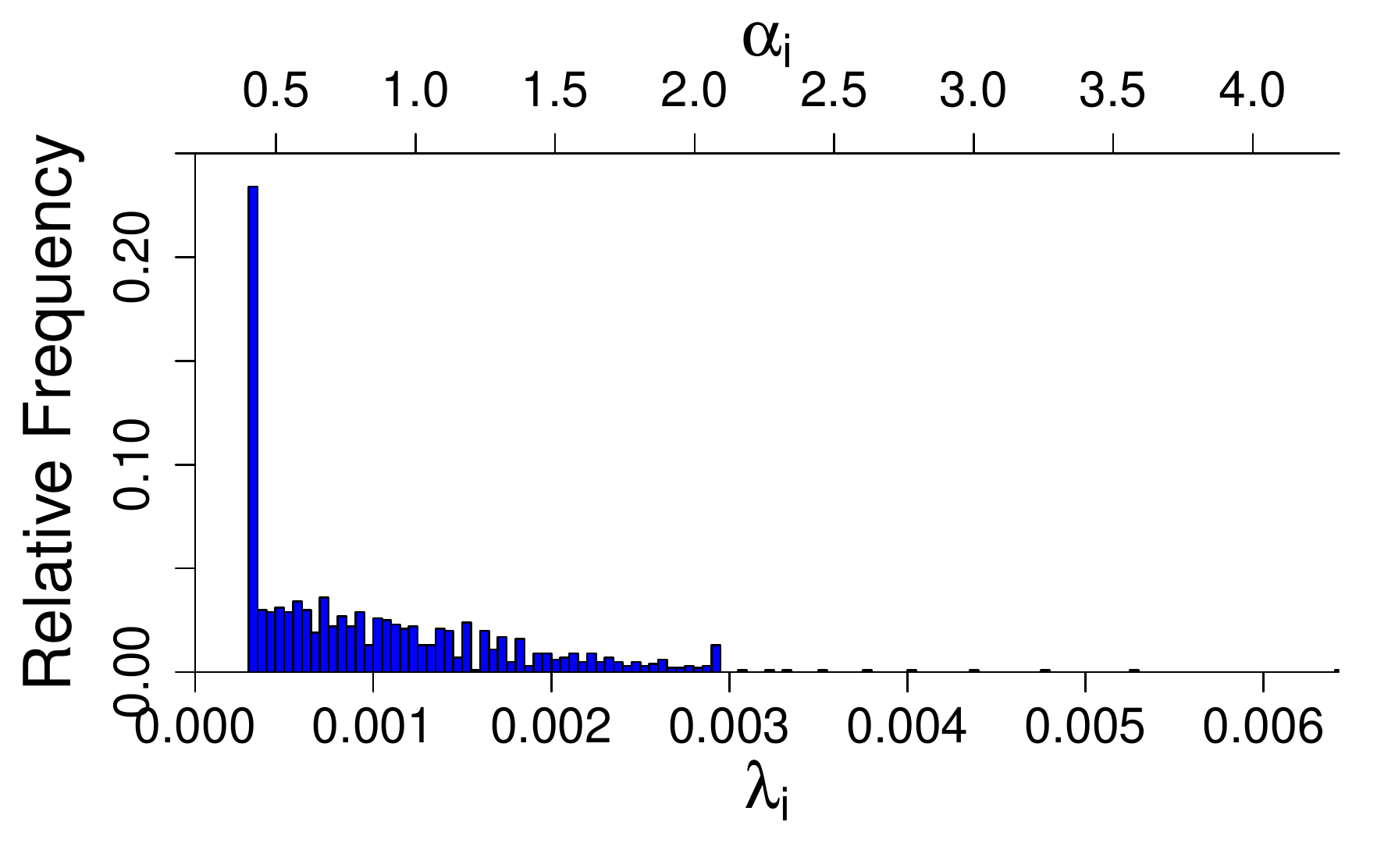}	
	\vspace{-0.5cm}
    \caption{$\X$ i.i.d. $\mathcal{N}(0,1/n), n=300, p=1000, \Pi$ is Bernoulli($\epsilon$), $\delta=n/p=0.3, \epsilon=0.5, \sigma_w=0$. 
		Top: A sample run of PGD that finds the minimizing $\bfalph$, \textit{not} the minimizing $\hat\beta$ for SLOPE in \cite{bu2019algorithmic}. Additionally, we plot two gradient descent methods with 0.9 momentum: the projected heavy ball (PHB) and the projected Nesterov accelerated gradient. All methods use the fine-tuned Lasso penalty as their starting points. 
		Top-middle: Red dashed line is Lasso MSE path (each point corresponds to one Lasso penalty, as $\lambda$ varies from 0 to large values); other lines are different SLOPE MSE for a single SLOPE penalty. BH here stands for ``Benjamini and Hochberg''.
		Bottom-middle: Best $\bfalph$ (right y-axis) and best $\blam$ (left y-axis) sequences found by PGD.
		Bottom-right: Histogram of best $\blam$ (bottom) and $\bfalph$ (upper) sequences found by PGD.}
\label{fig:optimal SLOPE}
\end{figure}

\subsection{Projected Gradient Descents}


Now that we have the gradients in Theorem \ref{thm:gradient} and the projection in Algorithm \ref{alg: projection}, the projected gradient descent is straight-forward. In each iteration we first conduct gradient descent using Theorem \ref{thm:gradient} and transform the sequence from $\bfalph$ regime to $\blam$ regime. The transformation $\Lambda$ is defined in Section 2.1 using the calibration and the state evolution in AMP. Then we project the gradient onto the constrained space $\mathcal{S}$ and transform it back to $\bfalph$ regime. The procedure is summarized in Algorithm \ref{alg: PGD}.
	
	

\begin{algorithm}[H]
\begin{algorithmic}
	\State\textbf{Input: } initial $\mathbf{\alpha}^{0}$, step size $\{s_t\}$
	\For{$t=1,\cdots,T$}
	\State $\triangleright$ Gradient descent on $\mathbf{\alpha}$ and transform to $\blam$ regime
	\State $\mathbf{\gamma}^{t+1} = \Lambda(\mathbf{\alpha}^{t} -s_t \nabla_{\mathbf{\alpha}}(\tau(\mathbf{\alpha}^t)))$
	\State $\triangleright$ Project onto $\mathcal{S}$
	\State $\blam^{t+1}=\texttt{ProjectOnS}(\bm\gamma^{t+1})$
	
	\State $\triangleright$ Transform back to $\bfalph$ regime
	
	\State $\mathbf{\alpha}^{t+1} = \Lambda^{-1}(\blam^{t+1})$

	\EndFor
	\State\textbf{Output: } $\mathbf{\alpha}^{T+1}$
\end{algorithmic}
	\caption{\texttt{Projected Gradient Descent} (PGD)}
	\label{alg: PGD}
\end{algorithm}
We highlight that Algorithm \ref{alg: PGD} is only one form of PGD. In fact, with a concrete form of the gradients, we can use any off-the-shelf first-order optimizer to find $\bfalph$ iteratively. Some examples include projected versions of stochastic gradient descent, Heavy Ball method \cite{polyak1964some}, Nesterov accelerated gradient descent \cite{nesterov1983method}, Adagrad \cite{duchi2011adaptive}, AdaDelta \cite{zeiler2012adadelta} and Adam \cite{kingma2014adam}. We include some of these optimizers in Figure \ref{fig:optimal SLOPE}.

To understand the convergence behavior of PGD, we need to study the convexity of the domain $\mathcal{S}$ and the objective function $\tau$. Clearly $\mathcal{S}$ is convex by simply applying the definition. Unfortunately, $\tau(\bfalph)$ for SLOPE is in general non-convex: even in the Lasso AMP regime where $\alpha\in\R$, it is shown that $\tau(\alpha)$ is only a quasi-convex function of $\alpha$ \cite[Theorem 3.3]{mousavi2018consistent}. As for SLOPE, the analysis on quasi-convexity of $\tau(\bfalph)$ has not been established. Yet, we do not observe any local minimum in practice. This is possibly the case since some non-convex problems may still enjoy desirable properties such as having unique global minimum or not having local minima. 

Remarkably, the gradient information that we use distinguishes our work from \cite{hu2019asymptotics}. We pause a bit and compare our approach with theirs, as they work under very similar assumptions as our AMP regime (in fact, both AMP and CGMT regimes agree asymptotically). Instead of optimizing directly on $\tau$, they propose to optimize the proximal operator $\eta$ in the functional space \cite[Proposition 3]{hu2019asymptotics}: for a fixed candidate $\tau$, they use the finite approximation with 2048 grids to solve a functional optimization, whose minimum is $\mathcal{L}(\tau)$. Next, they check the feasibility of the candidate $\tau$ by whether $\mathcal{L}(\tau) \leq \delta(\tau^2-\sigma_w^2)$. Lastly, a binary search is conducted to find the optimal $\tau$ (smallest feasible $\tau$) and the optimal design can be derived from the corresponding $\eta$. In summary, this approach took a detour by using a zeroth-order optimization algorithm, as the authors did not search over $\blam$ (or $\bfalph$) directly. Our first-order algorithm overcomes the seemingly unwieldy computation burden, especially in the high dimension when $p$ is very large.

\subsection{Transforming from $\alpha$ to $\lambda$}
Once we find the desirable $\bfalph$ with Algorithm \ref{alg: PGD}, the calibration \eqref{eq:calibration} allows us to convert $\bfalph$ to $\blam$ in the original SLOPE problem. We demonstrate in Figure \ref{fig:optimal SLOPE} and Figure \ref{fig:SLOPE vs MMSE} that SLOPE can outperform the best-tuned Lasso significantly. In Figure 1, SLOPE reduces $\MSE(\bet,\bethat)$ from 0.473 by Lasso to 0.350 by SLOPE, a 26\% improvement in the estimation error. In fact, we observe from Figure \ref{fig:SLOPE vs MMSE} that SLOPE is even comparable to Minimum Mean Squared Error (MMSE; proposed by \cite{bayati2011dynamics}) estimator, which produces the lowest MSE possible. We emphasize that our result does not contradict \cite{wang2019does} which states that, under some conditions, the Lasso is the optimal SLOPE. We note that the condition in \cite[Theorem 2]{wang2019does} does not hold for large systems: the premise of Lasso being optimal is that the Lasso achieves exact recovery, which requires $n\sim p\log p$ (see \cite{wainwright2009sharp}). Therefore, in our setting where $n/p\to\delta$, the Lasso is incapable of achieving the exact recovery nor outperforming general SLOPE.

\section{$k$-level SLOPE}
In this section we propose a method, described in Algorithm \ref{alg: PCD}, that works on the general linear model. I.e. our method works on arbitrary data $\X,\y,\bm\beta,\bm w$ and does not require $n/p\to\delta$ when $n,p\to\infty$. 

In contrast to the AMP regime, we directly search on $\blam$ without implicitly using $\bfalph$ and we do not try to use the gradient information. To avoid searching in the high dimension of $\blam$ space, we propose to restrict that the penalty $\blam$ only contains $k$ different non-negative values, which is denoted by $(\lambda_1,\cdots,\lambda_k;S_1,\cdots,S_{k-1})$. Here $\lambda_i$ denotes the penalty magnitude and $S_i$ represents the splitting index in $[p]$, where the penalty magnitudes change, i.e, $S_{i}-S_{i-1}$ entries in $\blam$ take the value $\lambda_i$. We note that $\lambda_i$ is decreasing in $i$ while $S_i$ is increasing, guaranteeing that $\blam$ satisfies the assumption of SLOPE. As an example in $\R^5$, $\blam=(7,5,1;2,3)=(7,7,5,1,1)$. We name this restricted SLOPE problem as the \textit{$k$-level SLOPE} and design the $(2k-1)$ degree of freedom penalty $\blam$, so as to only search in the reduced dimension $k \ll p$.

Notice that the original SLOPE is the $p$-level SLOPE and the Lasso is the 1-level SLOPE. We note that $k$-level SLOPE is always a sub-case of $(k+1)$-level SLOPE. Therefore intuitively, by allowing $k$ to take values other than 1 and $p$, we can trade off the difficulty of designing the penalty and the accuracy gain by employing more penalty levels. We demonstrate that empirically, the trade-off is surprisingly encouraging: even 2 or 3 levels of penalty is sufficient to exploit the benefit of SLOPE.



\subsection{Practical penalty design for $k$-level SLOPE}
We emphasize that in the general regime beyond AMP and CGMT, we cannot access the gradient information nor the functional optimization in \cite{hu2019asymptotics} for two reasons: the true $\bm\beta$ distribution is not known in real data and the data matrix $\X$ is general (not i.i.d. with a specific variance). To design the $k$-level SLOPE penalty in the real-world datasets, we propose the Coordinate Descent (CD, Algorithm \ref{alg: PCD})\footnote{We slightly abuse the notation of \MSE{ } to mean either the estimation error (only available in synthetic data) or the prediction error.} and compare to the PGD in Algorithm \ref{alg: PGD} under the AMP and CGMT regimes in Figure \ref{fig:SLOPE vs MMSE}.

\begin{algorithm}[!htb]
\begin{algorithmic}
\State\textbf{Input:} initial $\blam, \MSE_{old}=\infty$, level $k$
\While{$\MSE<\MSE_{old}$}
\State Set $\MSE_{old}=\MSE$
\For{$i \in \{1,\dots,k\}$} 
\State $\triangleright$ Search on magnitudes $\lambda_i$
\State$\lambda_i = \underset{\lambda_i\in(\lambda_{i+1},\lambda_{i-1})}{\textup{argmin}} \MSE(\lambda_1,\cdots;S_1,\cdots)$
\EndFor
\For{$i \in \{1,\dots,k-1\}$} 
\State $\triangleright$ Search on splits $S_i$
\State$S_i = \underset{S_i\in(S_{i-1},S_{i+1})}{\textup{argmin}} \MSE(\lambda_1,\cdots;S_1,\cdots)$
\EndFor
\EndWhile
\State\textbf{Output: $\blam=(\lambda_1,\cdots,\lambda_k;S_1,\cdots,S_{k-1})$}
\end{algorithmic}
\caption{\texttt{Coordinate Descent} (CD)}
\label{alg: PCD}
\end{algorithm}

We highlight some details of Algorithm \ref{alg: PCD} that make it efficient and practical. First of all, Algorithm \ref{alg: PCD} directly works on $\blam$ instead of $\bfalph$ (the calibration is generally unavailable). Second, the projection is not needed as in Algorithm \ref{alg: PGD} since $\blam$ is decreasing and non-negative by our definition of the search domain. Third, Algorithm \ref{alg: PCD} is flexible in the following sense: (1) we can choose any order of coordinates to successively minimize the error, e.g. by $\lambda_1,S_1,\lambda_2,S_2,\cdots$; (2) we can use any zeroth-order search method such as the grid search or the binary search for the magnitudes and splits.

\section{Experiments}

In this section, we justify the effectiveness of $k$-level SLOPE on various synthetic and real datasets, on linear and logistic regression tasks. For the sake of implementation consistency, we adopt R package \texttt{SLOPE} to run both Lasso and SLOPE in experiments. Empirically, we remark that PGD and $k$-level CD are both significantly fast in all experimental settings, taking only a few minutes to converge even for $p=1000$.







\subsection{Synthetic datasets}

\subsubsection*{Independent case}
In this experiment we investigate the performance of $k$-level SLOPE in the AMP regime: data matrix $\X$ is i.i.d. $\mathcal{N}(0,1/n), n=300$. The signal distribution $\Pi$ is Gaussian-Bernoulli with probability $0.5$ being standard normal and 0 otherwise. We work on a high-dimensional setting where $\delta=n/p=0.3$ and hence $\X\in\R^{300\times 1000}$. The noise $\sigma_w$ is 0. We observe from Figure \ref{fig:SLOPE vs MMSE} that employing more levels of penalty is beneficial and fast, suggesting that even a small $k$ may be sufficient to reduce the errors significantly.

\begin{figure}[!htp]
	\centering
	\includegraphics[width=7cm]{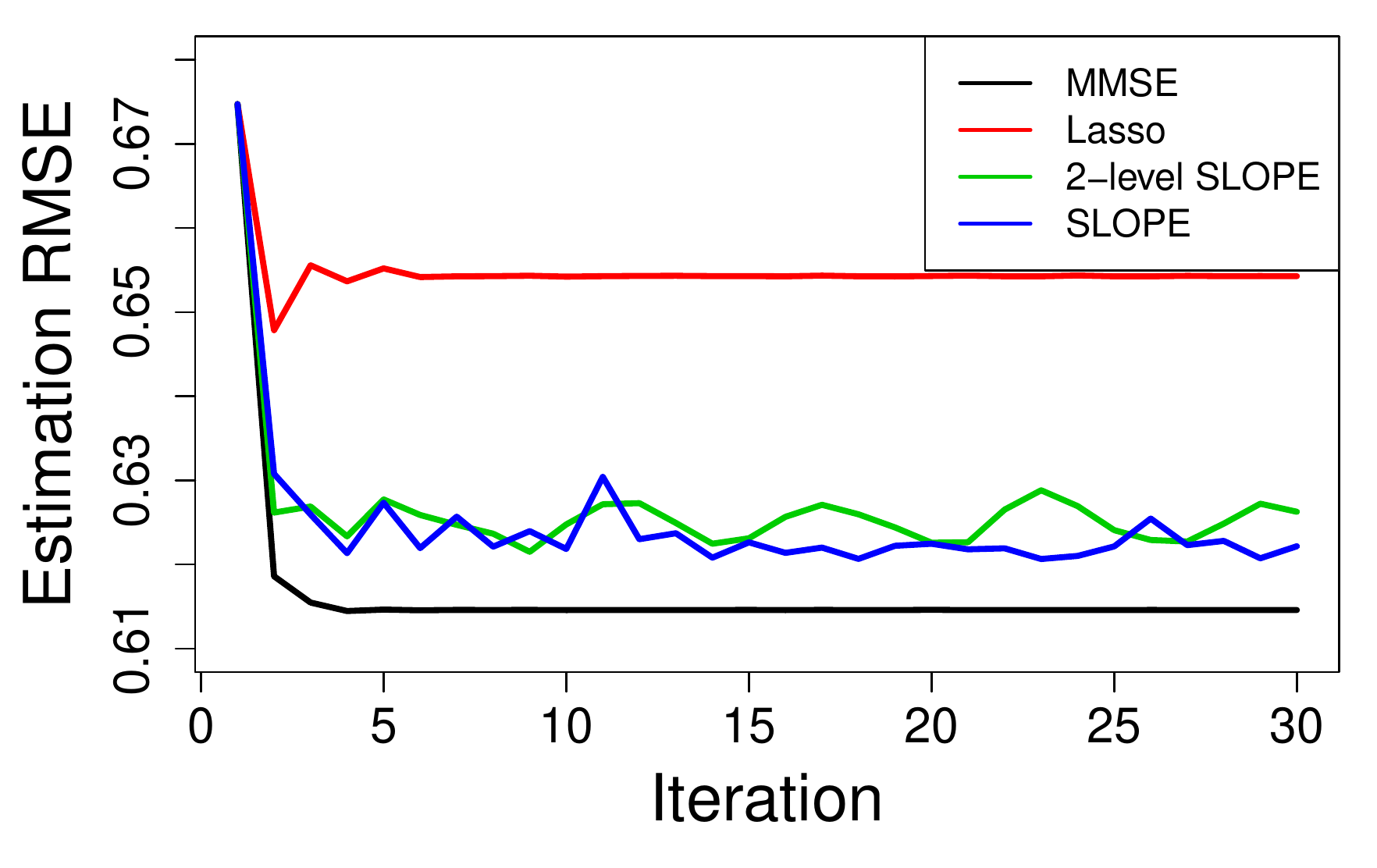}
	\includegraphics[width=7cm]{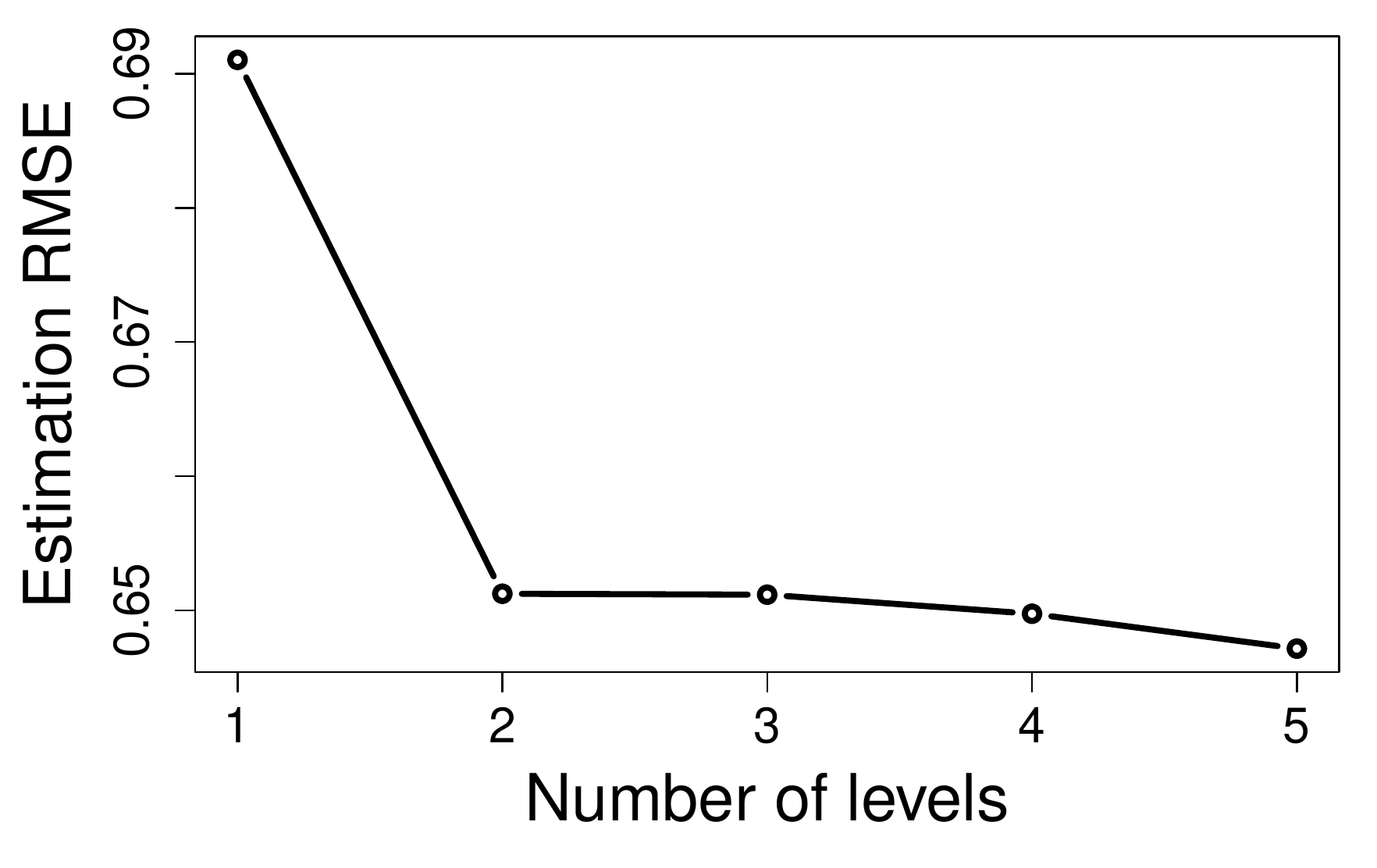}
	\caption{Top: the result of a single run for MMSE AMP, Lasso AMP, 2-level SLOPE AMP (by CD) and the $p$-level SLOPE AMP (by PGD). Bottom: averaged result over 10 independent runs of different $k$-level SLOPE by CD.}
\label{fig:SLOPE vs MMSE}
\end{figure}

\subsubsection*{Dependent case}
Different from the AMP regime in which each entry in the design matrix is i.i.d. gaussian, we study the performance of 2-level SLOPE in a synthetic dataset with features strongly correlated with each other. We include three other methods: Lasso and SLOPE with two other designs: Benjamini Hochberg design and the MR design proposed in \cite{bellec2018slope}. The data $\X$ is generated from an ARMA(1,1) model:
\begin{equation}\label{arma}
    X_{t}=\varepsilon_{t}+ 0.8 X_{t-1}+ 0.8 \varepsilon_{t-1}
\end{equation}
where $X_t$ denote the $t$-th feature and $\varepsilon_{t}$ follows i.i.d. $ \mathcal{N}(0,1)$. We set $\Pi$, the asymptotic distribution of $\bet$, to be i.i.d. Gaussian-Binomial: $\bet_i \sim \B\Z$ with $\B \sim B(5,0.3)$ and $\Z$ being standard normal. In terms of the dimension, we study two cases (1) $n = 20$, $p = 50$; (2) $n = 200$, $p = 500$. 10-fold cross-validation $\MSE(\y,\hat\y)$ are calculated for both the Lasso and SLOPE. We highlight that, different than the previous section, we investigate the prediction error instead of the estimation error here. Curves for $\MSE(\y,\hat\y)$ with different iterations in both cases are shown in Figure \ref{fig: synthetic}. In the first case, using grid search, the optimal prediction $\MSE(\y,\hat\y)$ given by Lasso is 0.128 while the optimal prediction $\MSE(\y,\hat\y)$ given by 2-level SLOPE (using Algorithm \ref{alg: PCD}) is 0.083. Prediction errors of SLOPE with other two penalty sequences are also under 0.1, but worse than that of 2-level SLOPE. We observe a $35\%$ improvement on prediction error when using 2-level SLOPE for this case of small sample size and dimension, compared with Lasso. In the second case, the optimal prediction $\MSE(\y,\hat\y)$ given by Lasso and other two SLOPEs are no smaller than 0.2, while that of $2$-level SLOPE is 0.186, giving a 7.5\% reduction in the prediction error.

\subsection{Real datasets for linear and logistic regression}
To further demonstrate the utility of $k$-level SLOPE in practice, we apply the model to real datasets, where $\MSE(\bm\beta,\hat{\bm\beta})$ is intractable, and focus on the prediction $\MSE(\y,\hat\y)$. In this experiment, again we compare the performance of 2-level SLOPE in a linear regression setting with three other methods we studied in Section 4.1. The dataset we adopt is atherosclerosis cardiovascular disease (ASCVD), which records medical information of 236 patients and their corresponding ASCVD risk score (outcome variable). We select 1000 features out of 4216 features, which have the largest correlation with the outcome variable. We conduct 20-fold cross-validation and calculate the cross-validation prediction $\MSE(\y,\hat\y)$. Using grid search, the optimal prediction $\MSE(\y,\hat\y)$ given by Lasso is 0.528. Interestingly, prediction $\MSE(\y,\hat\y)$ given by other two SLOPEs are worse than that of Lasso while that given by 2-level SLOPE (using Algorithm \ref{alg: PCD}) is 0.489. This result clearly demonstrates the outperformance of $k$-level SLOPE compared to Lasso and SLOPE using other penalty sequences. A curve for $\MSE(\y,\hat\y)$ with different iterations is shown in Figure \ref{fig: synthetic}.

We further extend the idea of $k$-level SLOPE in logistic regression and justify the results on Alzheimer's Disease Neuroimaging Initiative (ADNI) gene dataset. The dataset contains over 19000 genomic features of 649 patients, along with a binary disease status (normal or ill). We select the first 300 patients in the original dataset and 500 features out of the total features, which has the largest correlation with the outcome variable. We conduct 10-fold cross-validation and calculate the cross-validation prediction accuracy. Using grid search, the optimal prediction accuracy given by Lasso is 0.62. The optimal prediction accuracy given by 2-level SLOPE (using Algorithm \ref{alg: PCD}) is 0.66.

\begin{figure}[!htp]
	\centering
	\includegraphics[width=7cm]{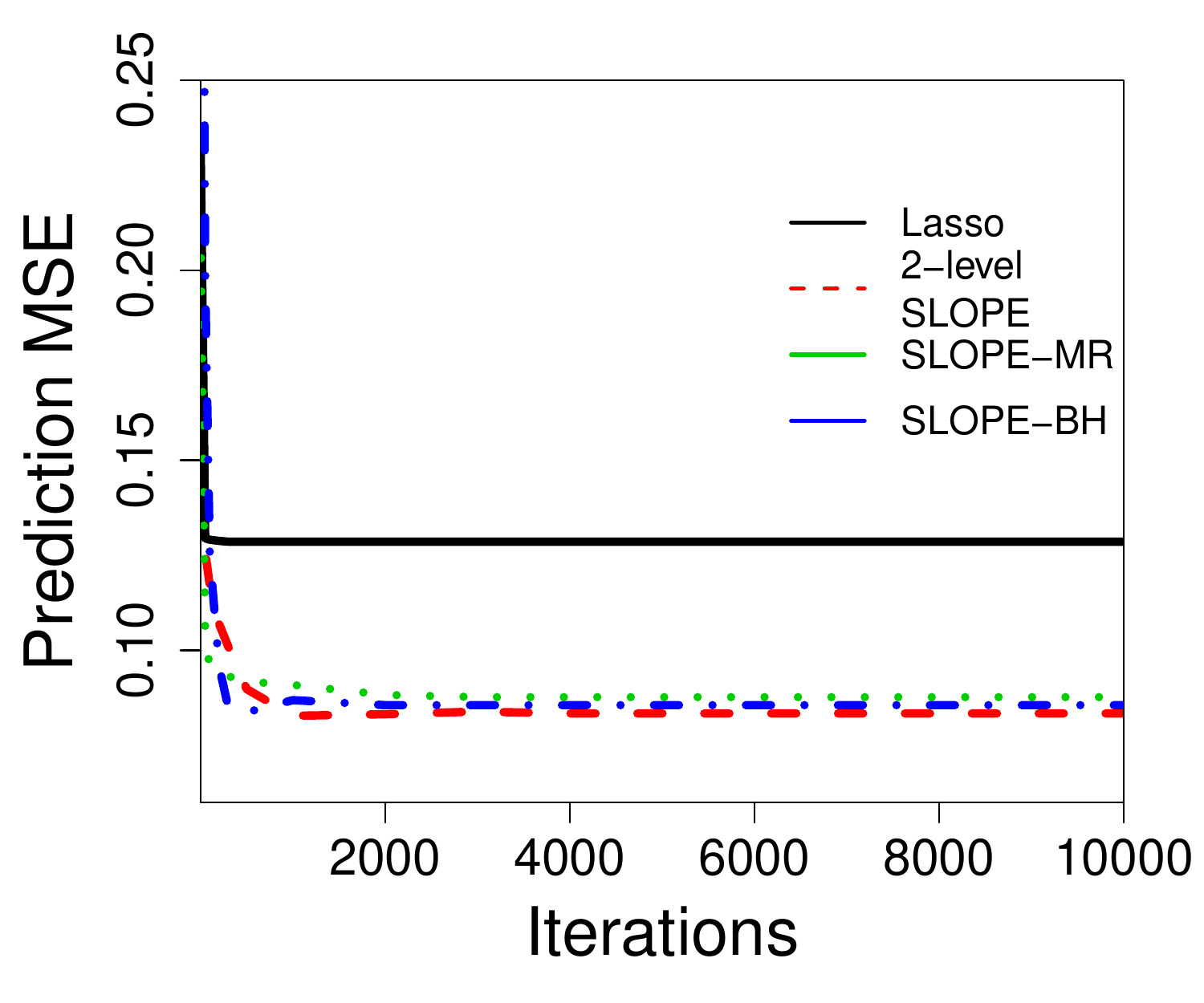}
	\includegraphics[width=7cm]{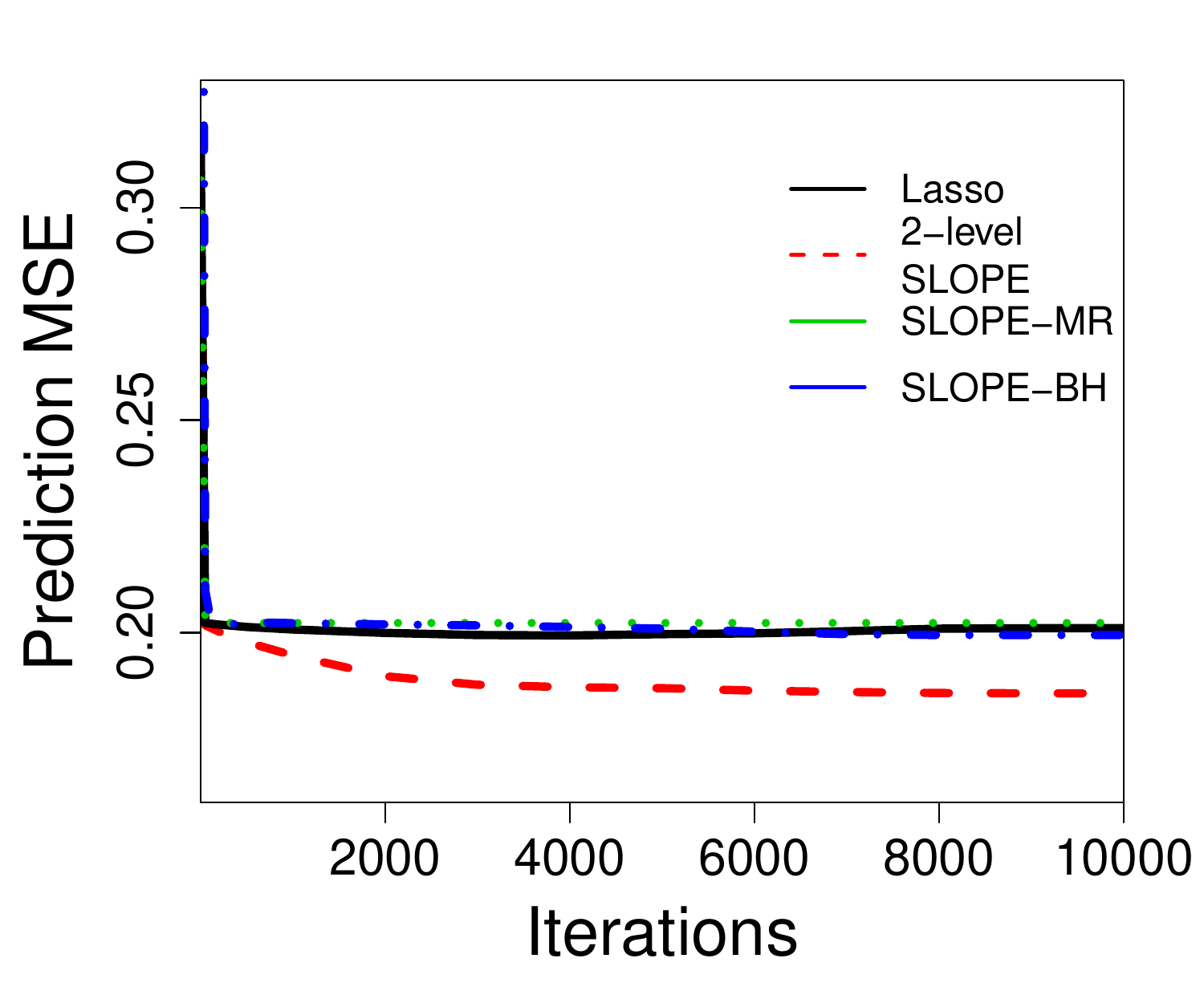}
	\includegraphics[width=7cm]{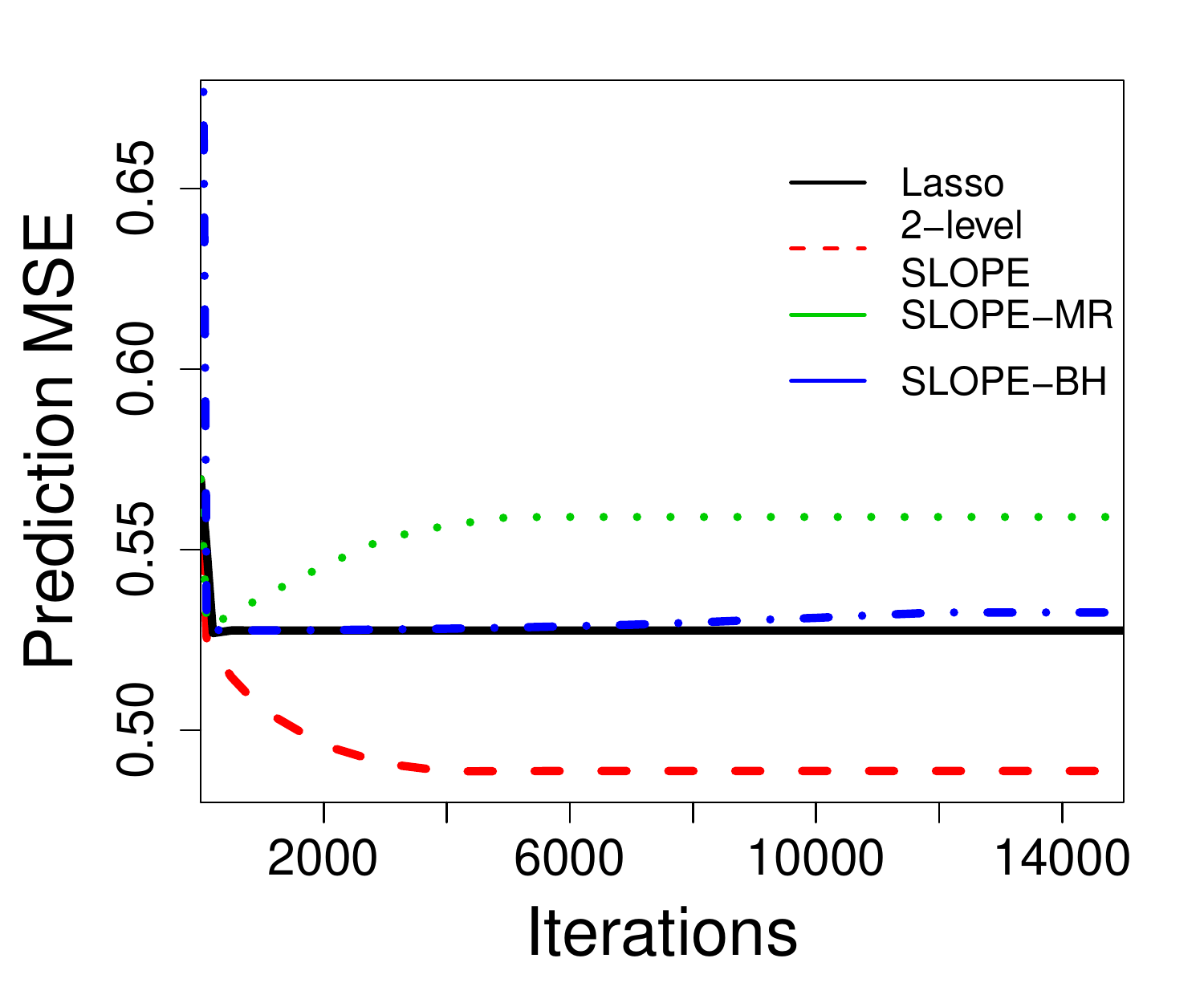}
	\caption{$\MSE(\y,\hat\y)$ in linear regression cases. SLOPE-MR: SLOPE using penalty sequence suggested in \cite{bellec2018slope}; SLOPE-BH: SLOPE using Benjamini–Hochberg penalty sequence in R function `SLOPE'. Top: Synthetic data with $\X$ i.i.d. drawn from ARMA(1,1) model (\ref{arma}), $n=20$, $p=50$. Middle: $\X$ i.i.d. drawn from (\ref{arma}) with $n=200$, $p=500$. Bottom: the results of ASCVD dataset.}
\label{fig: synthetic}
\end{figure}



\section{Discussion}
In this work, we propose a framework to flexibly and efficiently design the SLOPE penalty sequence. Under the AMP setting, our first-order PGD approach is capable of finding the effective penalty sequence with reasonable computation budget. The key is to use the gradient with respect to the penalty instead of using zeroth-order search as previous works have proposed. In the practical world beyond the AMP setting, via various experiments, we illustrate that the proposed $k$-level SLOPE with penalty sequence determined by Algorithm \ref{alg: PCD} can provide decent results. Although Algorithm \ref{alg: PCD} loses the access to the first-order information when compared to Algorithm \ref{alg: PGD}, the universal ability to search good penalty is desirable for practical use, as we can view the algorithm as a dimension-reduction trick. In many cases even 2-level SLOPE, the simplest $k$-level SLOPE (other than Lasso), can outperform the Lasso in accuracy as well as the ($p$-level) SLOPE in computation speed. Additionally, our framework indeed generalizes to other high-dimensional penalty designs. Some direct extensions include group SLOPE and weighted Lasso.

Much room is left for future study. From a theoretical perspective, the quasi-convexity of $\tau(\alpha)$ in AMP setting is still not well studied. The asymptotic $\MSE{(\bet,\hat\bet)}$ (i.e. Equation (7) in \cite{mousavi2018consistent}) is shown to be quasi-convex in Lasso case. However, no such theoretical property has been shown for SLOPE. If the quasi-convexity indeed holds true for SLOPE AMP, then we can guarantee that the minimizing $\blam$ by PGD is indeed the global minimizer and thus claim our design is optimal. 

It would also be interesting to develop PGD (based on AMP regime) for $k$-level SLOPE, i.e. using gradient descent to find the optimal magnitudes and splits. One could then derive a theoretical trade-off curve between the minimum $\tau$ and each $k$, similarly to Figure 2 bottom subplot. This would suggest a proper choice of $k$ for our $k$-level SLOPE.

From a practical perspective, we anticipate that $k$-level SLOPE can also be explored in various applications that already employ the Lasso, such as the matrix completion, the compressed sensing and the neural network regularization.





\bibliographystyle{plain} 
\bibliography{ref.bib}

\clearpage

\onecolumn

\appendix
\aistatstitle{Supplement Material for `Efficient Designs of SLOPE Penalty Sequences in Finite Dimension'}
\section{Introduction to MMSE AMP}

We firstly introduce the procedure for general AMP procedure. 
\begin{equation}
\begin{aligned}
& s^{(t+1)} = X^{\top}\Z^{(t)} + \bet^{(t)} \\
& \bet^{(t+1)} = \eta^{(t+1)}(s^{(t+1)})\\
& \Z^{(t+1)} = y - X\bet^{(t+1)} + \frac{1}{n}\Z^{(t)}[\nabla \eta^{(t)}(s^{(t)})]
\end{aligned}
\end{equation}
Different $\eta$ functions give different AMP, e.g. the soft-thresholding $\eta$ gives the Lasso AMP; the SLOPE proximal operator $\eta$ gives the SLOPE AMP.

The MMSE AMP adopts the following denoiser $\eta^{(t)}$ \cite{bayati2011dynamics}
$$
\eta_i^{(t)}(s) = \mathbb{E}[\bet|\bet + \tau_t\z = s_i] \text{ \ \ \ $i$ = 1, \dots, $p$}
$$
with $\z \sim \mathcal{N}(0,1)$. In above, using the state evolution \cite{bu2019algorithmic}, $\tau_t^2$ can be calculated iteratively as:
$$
\tau_t^2 = \sigma_{\omega}^2 + \frac{1}{\delta}\mathbb{E}[(\eta^{(t-1)}(\bet + \tau_{t-1}\z) - \bet)^2]
$$

Assume that the measurement matrix $X$ has i.i.d. $\mathcal{N}(0,1/n)$ entries. In many scenarios, the denoiser $\eta^{(t)}$ might be hard to calculate. Here we provide a derivation about calculating $\eta^{(t)}$ in the Bernoulli-Gaussian case: we assume that true signal $\bet \overset{i.i.d.}{\sim} \B$ where $\B$ is a Bernoulli-Gaussian distribution, i.e. $\bet_i = 0$ with probability $e \in [0,1]$, otherwise $\bet_i \sim \mathcal{N}(0, \sigma^2_{\B})$. 
\begin{equation}\label{eq: denoiser}
\mathbb{E}[\bet|\bet + \tau_t\z = s_i] = \mathbb{E}[\bet|\bet \neq 0, \bet + \tau_t\z = s_i]\mathbb{P}(\bet \neq 0| \bet + \tau_t\z = s_i)
\end{equation}
It's straightforward to see that, with $f$ denoting the corresponding probability density function,
\begin{equation}\label{eq:A}
\mathbb{P}(\bet \neq 0| \bet + \tau_t\z = s_i) = \frac{f(\bet + \tau_t\z = s_i|\bet \sim \mathcal{N}(0,\sigma^2_{B}))(1-e)}{f(\bet + \tau_t\z = s_i|\bet \sim \mathcal{N}(0,\sigma^2_{B}))(1-e) + f(\tau_t\z = s_i)e}
\end{equation}
Meanwhile. we have
$$
\mathbb{E}[\bet|\bet \neq 0, \bet + \tau_t\z = s_i] = \mathbb{E}[\bet|\bet \sim \mathcal{N}(0,\sigma^2_{B}), \bet + \tau_t\z = s_i]
$$
since $\bet + \tau_t\z \sim \mathcal{N}(0,\sigma^2_{B}+\tau_t^2)$, conditional expectation on joint normal distribution yields
\begin{equation}\label{eq:B}
\mathbb{E}[\bet|\bet \sim \mathcal{N}(0,\sigma^2_{B}), \bet + \tau_t\z = s_i] = \frac{\sigma^2_{B}}{\sigma^2_{B} + \tau^2_t}s_i
\end{equation}
(\ref{eq:A}) and (\ref{eq:B}) give a simple way to calculate the denoiser using (\ref{eq: denoiser}).

\newpage

\section{Analysis of Gradient in PGD for $\alpha$} \label{a1}

\begin{proof}[Proof of Theorem \ref{thm:gradient}]

Minimizing the estimation error is equivalent to minimizing $\tau$. Since the AMP algorithms are working on the finite dimension, we analyze the finite-size approximation of the state evolution \cite[Equation (2.5)]{bu2019algorithmic}:
$$\tau^2=\sigma_w^{2}+\frac{1}{\delta p} \mathbb{E}\left\|\operatorname{prox}_{J_{\bfalph\tau}}\left(\bet+\tau \Z\right)-\bet\right\|^{2}$$
In finite dimensions, the expectation is taken with respect to $\Z$. Differentiating both sides of the state evolution with respect to $\alpha_i$ and denoting $\tau'=\frac{\partial \tau}{\partial\alpha_i}$ gives:
\begin{align}
2\tau\tau'&=\frac{\partial}{\partial\alpha_i}\left(\sigma_w^2+\frac{1}{\delta p}\E\|\prox_{J_{\bm\alpha\tau}}(\bet+\tau \Z)-\bet\|^2\right) \nonumber\\
&=\frac{1}{\n}\frac{\partial}{\partial\alpha_i}\sum_{j=1}^{p}\E\left([\prox_{J_{\bm\alpha\tau}}(\bet+\tau \Z)]_j-\bet_j\right)^2
\end{align}
Recall $\eta_j$ represents the $j$-th element of $\Eta := \prox_{J_{\bfalph\tau}}(\boldsymbol{\beta}+\tau\boldsymbol{Z})$. By chain rule
\begin{align}
2\tau\tau'&=\frac{2}{\n}\sum_{j=1}^{p}\E(\eta_j-\beta_j)\frac{\partial\eta_j}{\partial \alpha_i} \nonumber =\frac{2}{\n}\sum_{j=1}^{p}\E(\eta_j-\beta_j) \left[\sum_{k=1}^{p}\frac{d \eta_j}{d a_k}\frac{\partial a_k}{\partial \alpha_i}+\frac{d \eta_j}{d b_k}\frac{\partial b_k}{\partial \alpha_i}\right]
\label{eq:tau'}
\end{align}
where we define $a_k:=\beta_k+\tau Z_k, b_k:=\alpha_k\tau$. To calculate the derivatives, we pause to discuss forms of general derivatives of $ \Eta(\mathbf{a}, \mathbf{b})$. Define
\begin{align}
\partial_1\Eta(\mathbf{a}, \mathbf{b}) := \text{diag}\Big[\frac{\partial}{\partial a_1}, \frac{\partial}{\partial a_2}, \ldots, \frac{\partial}{\partial a_{\p}}\Big]\Eta(\mathbf{a}, \mathbf{b})\\
\partial_2\Eta(\mathbf{a}, \mathbf{b}) := \text{diag}\Big[\frac{\partial}{\partial b_1}, \frac{\partial}{\partial b_2}, \ldots, \frac{\partial}{\partial b_{\p}}\Big]\Eta(\mathbf{a}, \mathbf{b}).
\end{align}
According to \cite[Proof of Fact 3.4]{su2016slope} and \cite[Proof of Theorem 1]{bu2019algorithmic}, we have 
$$
[\partial_1 \Eta(\mathbf{a}, \mathbf{b})]_j=\frac{1}{\text{\#\{$1\leq k\leq \p: | [ \Eta(\mathbf{a}, \mathbf{b})]_k  |=| [ \Eta(\mathbf{a}, \mathbf{b})]_j  |$\}}}
$$  
and that 
\begin{align}
 \frac{d}{da_k}  [\Eta(\mathbf{a}, \mathbf{b})]_j 
 =& \mathbb{I} \{ |\Eta(\mathbf{a}, \mathbf{b})|_j = | \Eta(\mathbf{a}, \mathbf{b})|_k\} \nonumber \sgn(\Eta_j\Eta_k)  [\partial_1 \Eta(\mathbf{a}, \mathbf{b})]_j
\end{align}
for the derivative regardng the first variable. Recall that the permutation $\sigma: \{1,\dots,p\} \to \{1,\dots,p\}$ is the inverse mapping for ranking of indices such that $|\Eta|_{(i)} =  |[ \Eta]_{\sigma(i)}|$. Similarly, according to \cite[Proof of Theorem 1]{bu2019algorithmic}:
\begin{align}
&\frac{d}{d b_k}  [\eta(\mathbf{a}, \mathbf{b})]_j  = - \sgn ([\eta(\mathbf{a}, \mathbf{b})]_{\sigma(k)}) \frac{d}{da_{\sigma(k)}}  [\Eta(\mathbf{a}, \mathbf{b})]_j \nonumber\\
&= \mathbb{I}\big\{ |\eta(\mathbf{a}, \mathbf{b})|_j =|\eta(\mathbf{a}, \mathbf{b})|_{\sigma(k)}\big\}\sgn \big(\eta_j \big)    \big[\partial_1 \eta(\mathbf{a}, \mathbf{b})\big]_j.
\end{align}

In addition to $I_j$ defined in Section 2, we let  $K_j:=\{k:|\eta_{\sigma(k)}|=|\eta_j|\}$, which is the set of ranking indices whose corresponding entries share the same absolute value with $\eta_j$. This notion will be used to replace the indicator term $\mathbb{I}\big\{ |\eta(\mathbf{a}, \mathbf{b})|_j =|\eta(\mathbf{a}, \mathbf{b})|_{\sigma(k)}\big\}$ above. We can rewrite \eqref{eq:tau'} as
\begin{align}
2\tau\tau'&=\frac{2}{\n}\sum_{j=1}^p\E(\eta_j-\beta_j) \left[\sum_{k\in I_j}\frac{d \eta_j}{d a_k}\frac{\partial a_k}{\partial \alpha_i}+\sum_{k\in K_j}\frac{d \eta_j}{d b_k}\frac{\partial b_k}{\partial \alpha_i}\right]
\nonumber\\
& =\frac{2}{\n}\sum_{j=1}^p\E(\eta_j-\beta_j)\sgn(\eta_j) \nonumber \Bigg[\frac{1}{|I_j|}\sum_{k\in I_j}\sgn(\eta_k)\frac{\partial a_k}{\partial \alpha_i} - \frac{1}{|K_j|}\sum_{k\in K_j}\frac{\partial b_k}{\partial \alpha_i}\Bigg] \nonumber\\
& =\frac{2}{\n}\sum_{j=1}^p\E(\eta_j-\beta_j)\sgn(\eta_j)\Bigg[\frac{1}{|I_j|}\sum_{k\in I_j}\sgn(\eta_k)Z_k \tau' \nonumber - \frac{1}{|K_j|}\sum_{k\in K_j}(\alpha_k \tau' + \mathbb{I}\{k=i\} \tau)\Bigg]
\end{align}
Merging the terms containing the derivative $\tau'$ on one side gives
\begin{align}
&\frac{1}{n}\sum_{j \in I_{\sigma(i)}}\E(\eta_j-\beta_j)\sgn(\eta_j)\tau\cdot\frac{1}{|K_j|} \nonumber\\
& = \frac{1}{n}\sum_{j=1}^p\E(\eta_j-\beta_j)\sgn(\eta_j) \nonumber \Bigg[\frac{1}{|I_j|}\sum_{k\in I_j}\sgn(\eta_k)Z_k \tau' - \frac{1}{|K_j|}\sum_{k\in K_j}\alpha_k \tau'\Bigg] - \tau \tau'
\end{align}
Notice that $|I_j| = |K_j|$ due to $\sigma$ being a permutation, we can simplify above as
\begin{align}
\frac{\partial\tau}{\partial\alpha_i}=\E\frac{1}{|I_{\sigma(i)}| D(\mathbf{\alpha},\tau)}\sum_{j \in I_{\sigma(i)}}(\eta_j-\beta_j)\sgn(\eta_j)\tau
\end{align}
where $D(\mathbf{\alpha},\tau)$ in the denominator is
\begin{align}
D(\mathbf{\alpha},\tau) &= -n\tau +  \sum_{j=1}^p \E\frac{1}{|I_j|}(\eta_j-\beta_j)\sgn(\eta_j) \nonumber \sum_{k \in I_j}(\sgn(\eta_k)Z_k - \alpha_{\sigma^{-1}(k)})
\end{align}
We next show that $D(\mathbf{\alpha},\tau)$ is always negative. Firstly observe from (\ref{eq:state evolution}) that
\begin{align}\label{ap:tau}
\tau^2 > \frac{1}{n}\sum_{j=1}^p\mathbb{E}(\eta_j - \beta_j)^2
\end{align}
Now for the set $I_i$ with a fixed index $i$,
\begin{align}
& \ \ \ \sum_{j \in I_i}(\eta_j - \beta_j)^2 \geq \frac{1}{|I_i|} (\sum_{j \in I_i} |\eta_j - \beta_j|)^2\\
& \geq \frac{1}{|I_i|} (\sum_{j \in I_i} (\eta_j - \beta_j)\sgn(\eta_j))^2\\
& = \frac{1}{|I_i|} \sum_{j \in I_i} (\eta_j - \beta_j)\sgn(\eta_j) \sum_{k \in I_i}\tau Z_k\sgn(\eta_k) - \alpha_{\sigma^{-1}(k)}\tau \\
& \geq \frac{\tau}{|I_i|} \sum_{j \in I_i} (\eta_j - \beta_j)\sgn(\eta_j) \sum_{k \in I_j} Z_k\sgn(\eta_k) - \alpha_{\sigma^{-1}(k)}
\end{align}
This in turn implies that
\begin{align}
& \ \ \ \sum_{j=1}^p (\eta_j - \beta_j)^2 = \sum_{j=1}^p \frac{1}{|I_j|}\sum_{k \in I_j}(\eta_k - \beta_k)^2 \geq \sum_{j=1}^p\frac{\tau}{|I_j|} (\eta_j - \beta_j)\sgn(\eta_j) \sum_{k \in I_j} Z_k\sgn(\eta_k) - \alpha_{\sigma^{-1}(k)}
\end{align}
Combining with (\ref{ap:tau}) yields $D < 0$.

\end{proof}

\section{Analysis of Projection in PGD for $\alpha$}
\subsection{Characterization of projection on $\mathcal{S}$} \label{a2}

We firstly prove that Algorithm \ref{alg: projection} indeed finds the projection. To do so we firstly provide a detailed characterization of the projection, then prove that the output of Algorithm \ref{alg: projection} matches the form of projection. We start by defining \textit{blocks} and \textit{segmentation blocks}, upon which our proof highly relies. Suppose $\mathbf{\gamma} = \{\gamma_1, \dots, \gamma_p\}$, \textit{blocks} are subsequences defined as $B(\mathbf{\gamma},u):= \{\gamma_u, \dots, \gamma_{u+L(\mathbf{\gamma},u)-1}\}$ where length $L(\mathbf{\gamma},u)$ is defined as
\begin{equation} \label{L}
\begin{aligned}
L(\mathbf{\gamma},u) = 
\begin{cases}
L^* & \text{if $L^* \neq \emptyset$}\\
p & \text{otherwise}
\end{cases}
\end{aligned}
\end{equation}
where
$$
\begin{aligned}
L^* \overset{\Delta}{=} & \min \Big\{ 1 \leq L \leq p-u \Big| \forall 0 \leq k \leq p-u-L, \frac{1}{k+1}\sum_{i=0}^{k} \gamma_{u+L+i} < \frac{1}{L}\sum_{i=0}^{L-1} \gamma_{u+i} \Big\}
\end{aligned}
$$
Roughly speaking, $L(\mathbf{\gamma},u)$ is the minimum value of a finite set (truncated at $p$ when the set is empty). For each element $L$ in this set, the average value in sequence $\{\gamma_u,\dots, \gamma_{u+L-1}\}$ is always larger than that of arbitrary sequence $\{\gamma_{u+L},\dots, \gamma_{u+L+k}\}$ whose left start is $\gamma_{u+L}$. With such definition of blocks, we can now segment $\mathbf{\gamma}$ into $q \leq p$ blocks:
$$
\begin{aligned}
\mathbf{\gamma} &= \{B(\mathbf{\gamma},1),B(\mathbf{\gamma},L(\mathbf{\gamma},1)+1), B(\mathbf{\gamma}, L(\mathbf{\gamma},L(\mathbf{\gamma},1)+1)+L(\mathbf{\gamma},1)+1),\dots \} \overset{\Delta}{=} \{B_1,\dots, B_q\}
\end{aligned}
$$
We call $B_1, \dots, B_q$ \textit{segmentation blocks} for vector $\mathbf{\gamma}$. It's straightforward to see that $B_k = B(\mathbf{\gamma},L_k)$ where $L_k$ satisfies $L_1 = L(\mathbf{\gamma},1)$ and 
$$
L_k = L(\mathbf{\gamma}, \sum_{i=1}^{k-1} L_{i} + 1)
$$
Our result shows that for input vector $\gamma$, its projection vector $\Pi_{\mathcal{S}}(\mathbf{\gamma})$ takes identical values inside each of the segmentation blocks. Before formally stating the theorem, We first highlight the following fact that will be frequently used in the proof of the theorem.
\begin{fact}\label{fct:1}
For two sequences of length $p$: $\{a_i\}$ and $\{b_i\}$, if $\sum a_i = \sum b_i$, then function $g(C):= \sum(b_i - a_i + C)^2$ is monotonically increasing with respect to $|C|$.
\end{fact}
\begin{proof}
Notice that 
$$
\begin{aligned}
\sum(b_i - a_i + C)^2 & = \sum(b_i - a_i)^2 + \sum 2C(b_i - a_i) + pC^2 = pC^2 + \sum(b_i - a_i)^2
\end{aligned}
$$
Hence $g(C)$ is is monotonically increasing with respect to $|C|$.
\end{proof}
\begin{theorem}
Let $B$ denote the segmentation block that contains $\gamma_i$, then
$$
(\Pi_{\mathcal{S}}(\mathbf{\gamma}))_i = \max \left\{ \frac{1}{|B|}\sum_{\gamma_j \in B} \gamma_j, 0 \right\}
$$
\end{theorem}

\begin{proof}

The proof consists of two steps. In the first step, we prove that for each segmentation block $B$, the projection of each coordinates share the same value. That is, $(\Pi_{\mathcal{S}}(\mathbf{\gamma}))_i = \mathcal{C}(B)$ as long as $\gamma_i \in B$. In the second step, we show that this constant is the mean of the block truncated at 0: $\mathcal{C}(B) = \max \left\{ \frac{1}{|B|}\sum_{\gamma_j \in B} \gamma_j, 0 \right\} $.

\textbf{Step 1} 
Without loss of generality, we consider $B = B(\mathbf{\gamma}, u)$.
We know from definition of blocks that $\forall 1 \leq l \leq L-1$, $\exists k_l$ s.t. $\frac{1}{k_l}\sum_{i=1}^{k_l} \gamma_{u+l-1+i} \geq \frac{1}{l}\sum_{i=1}^l \gamma_{u+i-1}$. We use induction to prove that $(\Pi_{\mathcal{S}}(\mathbf{\gamma}))_u = (\Pi_{\mathcal{S}}(\mathbf{\gamma}))_{u+l}$, $\forall 1 \leq l \leq L(\mathcal{\gamma},u)-1$. For $l = 1$, assume $(\Pi_{\mathcal{S}}(\mathbf{\gamma}))_u > (\Pi_{\mathcal{S}}(\mathbf{\gamma}))_{u+1}$. Consider two cases: (i) $(\Pi_{\mathcal{S}}(\mathbf{\gamma}))_{u} > \gamma_{u}$. (ii) $(\Pi_{\mathcal{S}}(\mathbf{\gamma}))_{u} \leq \gamma_{u}$. We now show that both cases lead to contradiction and hence do not hold. In case (i), we consider 
$$
(\widetilde{\Pi}_{\mathcal{S}}(\mathbf{\gamma}))_i = 
\begin{cases}
      \max\{\gamma_{u}, (\Pi_{\mathcal{S}}(\mathbf{\gamma}))_{u+1}\} & \text{if $i = u$ }\\
      (\Pi_{\mathcal{S}}(\mathbf{\gamma}))_{i} & \text{otherwise}
\end{cases}
$$
then obviously, 
$$
\left|(\widetilde{\Pi}_{\mathcal{S}}(\mathbf{\gamma}))_u - \gamma_{u}\right| < \left|({\Pi}_{\mathcal{S}}(\mathbf{\gamma}))_u - \gamma_{u}\right|
$$
which leads to that $\frac{1}{2}\|(\widetilde{\Pi}_{\mathcal{S}}(\mathbf{\gamma})) - \mathbf{\gamma}\|^2_2 < \frac{1}{2}\|({\Pi}_{\mathcal{S}}(\mathbf{\gamma})) - \mathbf{\gamma}\|^2_2$. This contradicts to the definition of projection. In case (ii), from definition of blocks we have that $\exists k_0 \geq 1$ s.t. $\frac{1}{k_0}\sum_{i=1}^{k_0} \gamma_{u+i} \geq \gamma_{u}$. Consider 
$$
(\widetilde{\Pi}_{\mathcal{S}}(\mathbf{\gamma}))_i = 
\begin{cases}
       (\Pi_{\mathcal{S}}(\mathbf{\gamma}))_{u} & \text{if $i \in \{u+1, \dots, u+k_0\}$ }\\
      (\Pi_{\mathcal{S}}(\mathbf{\gamma}))_{i} & \text{otherwise}
\end{cases}
$$
Notice that $\frac{1}{k_0}\sum_{i=1}^{k_0} \gamma_{u+i} \geq \gamma_{u} \geq (\Pi_{\mathcal{S}}(\mathbf{\gamma}))_{u} > (\Pi_{\mathcal{S}}(\mathbf{\gamma}))_{u+1}$, we have for $i \in \{u+1, \dots, u+k_0\}$, $\left|(\widetilde{\Pi}_{\mathcal{S}}(\mathbf{\gamma}))_i - ({\Pi}_{\mathcal{S}}(\mathbf{\gamma}))_i\right|$ is a constant independent of $i$ and that
$$
\left|(\widetilde{\Pi}_{\mathcal{S}}(\mathbf{\gamma}))_i - \frac{1}{k_0}\sum_{i=1}^{k_0} \gamma_{u+i}\right| < \left|({\Pi}_{\mathcal{S}}(\mathbf{\gamma}))_i - \frac{1}{k_0}\sum_{i=1}^{k_0} \gamma_{u+i}\right|
$$
According to Fact \ref{fct:1}, we define substitution for $i \in \{u+1, \dots, u+k_0\}$: $b_i = \frac{1}{k_0}\sum_{i=1}^{k_0} \gamma_{u+i}$, $a_i = \gamma_{u+i}$, $b_i + C_1 = (\widetilde{\Pi}_{\mathcal{S}}(\mathbf{\gamma}))_i$ and $b_i + C_2 = ({\Pi}_{\mathcal{S}}(\mathbf{\gamma}))_i$. Then since $|C_1 < C_2|$, we have $\frac{1}{2}\|(\widetilde{\Pi}_{\mathcal{S}}(\mathbf{\gamma})) - \mathbf{\gamma}\|^2_2 < \frac{1}{2}\|({\Pi}_{\mathcal{S}}(\mathbf{\gamma})) - \mathbf{\gamma}\|^2_2$, which contradicts to the definition of projection. 

Now assume the statement holds for $1 \leq l \leq l_0-1$, that is $(\Pi_{\mathcal{S}}(\mathbf{\gamma}))_u = \cdots = (\Pi_{\mathcal{S}}(\mathbf{\gamma}))_{u+l_0-1}$, we want to prove that $(\Pi_{\mathcal{S}}(\mathbf{\gamma}))_u = (\Pi_{\mathcal{S}}(\mathbf{\gamma}))_{u+l_0}$. Since the projection is on $\mathcal{S}$, by definition we know $(\Pi_{\mathcal{S}}(\mathbf{\gamma}))_u$ can never be smaller than $(\Pi_{\mathcal{S}}(\mathbf{\gamma}))_{u+l_0}$. We now assume $(\Pi_{\mathcal{S}}(\mathbf{\gamma}))_u > (\Pi_{\mathcal{S}}(\mathbf{\gamma}))_{u+l_0}$ and consider two cases: (i) $(\Pi_{\mathcal{S}}(\mathbf{\gamma}))_{u} > \frac{1}{l_0}\sum_{i=0}^{l_0-1} \gamma_{u+i}$. (ii) $(\Pi_{\mathcal{S}}(\mathbf{\gamma}))_{u} \leq \frac{1}{l_0}\sum_{i=0}^{l_0-1} \gamma_{u+i}$. To complete the proof, it suffices for us to show that neither of the cases can hold without contradictions. In case (i), we consider
$$
(\widetilde{\Pi}_{\mathcal{S}}(\mathbf{\gamma}))_i = 
\begin{cases}
      \max\{ \frac{1}{l_0}\sum_{j=0}^{l_0-1} \gamma_{u+j}, (\Pi_{\mathcal{S}}(\mathbf{\gamma}))_{u+l_0}\} \\
      \hfill \text{if $i \in \{u, \dots, u+l_0-1\}$ }\\
      (\Pi_{\mathcal{S}}(\mathbf{\gamma}))_{i} \hfill \text{otherwise}
\end{cases}
$$
then obviously for $i \in \{u, \dots, u+l_0-1\}$, $\left|(\widetilde{\Pi}_{\mathcal{S}}(\mathbf{\gamma}))_i - ({\Pi}_{\mathcal{S}}(\mathbf{\gamma}))_i\right|$ is a constant independent of $i$ and that 
$$
\left|(\widetilde{\Pi}_{\mathcal{S}}(\mathbf{\gamma}))_i - \frac{1}{l_0}\sum_{i=0}^{l_0-1} \gamma_{u+i}\right| < \left|({\Pi}_{\mathcal{S}}(\mathbf{\gamma}))_i - \frac{1}{l_0}\sum_{i=0}^{l_0-1} \gamma_{u+i}\right|
$$
According to Fact \ref{fct:1}, using the same substitution as that in analysis of $l=1$, we have that $\frac{1}{2}\|(\widetilde{\Pi}_{\mathcal{S}}(\mathbf{\gamma})) - \mathbf{\gamma}\|^2_2 < \frac{1}{2}\|({\Pi}_{\mathcal{S}}(\mathbf{\gamma})) - \mathbf{\gamma}\|^2_2$, which makes contradiction to the definition of projection. In case (ii), from definition of blocks we have that $\exists k_0 \geq 1$ s.t. $\frac{1}{k_0}\sum_{i=1}^{k_0} \gamma_{u+l_0-1+i} \geq \frac{1}{l_0}\sum_{i=0}^{l_0-1} \gamma_{u+i}$. Now we consider 
$$
(\widetilde{\Pi}_{\mathcal{S}}(\mathbf{\gamma}))_i = 
\begin{cases}
       (\Pi_{\mathcal{S}}(\mathbf{\gamma}))_{u} & \text{if $i \in \{u+l_0, \dots, u+l_0-1+k_0\}$ }\\
      (\Pi_{\mathcal{S}}(\mathbf{\gamma}))_{i} & \text{otherwise}
\end{cases}
$$
Notice that $\frac{1}{k_0}\sum_{i=1}^{k_0} \gamma_{u+l_0-1+i} \geq \frac{1}{l_0}\sum_{i=0}^{l_0-1} \gamma_{u+i} \geq (\Pi_{\mathcal{S}}(\mathbf{\gamma}))_{u} > (\Pi_{\mathcal{S}}(\mathbf{\gamma}))_{u+l_0}$, we have for $i \in \{u+l_0, \dots, u+l_0-1+k_0\}$, $\left|(\widetilde{\Pi}_{\mathcal{S}}(\mathbf{\gamma}))_i - ({\Pi}_{\mathcal{S}}(\mathbf{\gamma}))_i\right|$ is a constant independent of $i$ and that
$$
\begin{aligned}
\left|(\widetilde{\Pi}_{\mathcal{S}}(\mathbf{\gamma}))_i - \frac{1}{k_0}\sum_{i=0}^{k_0-1} \gamma_{u+l_0+i}\right| < \left|({\Pi}_{\mathcal{S}}(\mathbf{\gamma}))_i - \frac{1}{k_0}\sum_{i=0}^{k_0-1} \gamma_{u+l_0+i}\right|
\end{aligned}
$$
Again according to Fact \ref{fct:1}, we have $\frac{1}{2}\|(\widetilde{\Pi}_{\mathcal{S}}(\mathbf{\gamma})) - \mathbf{\gamma}\|^2_2 < \frac{1}{2}\|({\Pi}_{\mathcal{S}}(\mathbf{\gamma})) - \mathbf{\gamma}\|^2_2$, which contradicts to the definition of projection. This implies that it can never happen that $(\Pi_{\mathcal{S}}(\mathbf{\gamma}))_u > (\Pi_{\mathcal{S}}(\mathbf{\gamma}))_{u+l_0}$, which completes the induction. We have proved that $(\Pi_{\mathcal{S}}(\mathbf{\gamma}))_u = \cdots = (\Pi_{\mathcal{S}}(\mathbf{\gamma}))_{u+L(\mathbf{\gamma},u)-1} \overset{\Delta}{=} \mathcal{C}(B(u))$ for each segmentation block $B(u)$ of vector $\mathbf{\gamma}$.

\textbf{Step 2}
Now we already know that inside each segmentation block, the projection of each coordinate is a constant $\mathcal{C}(B)$, we now optimize the sequence $\{\mathcal{C}(B_i)\}_{i=1}^q$. According to Fact \ref{fct:1}, inside each $B_i$, the optimal constant (i.e. constant gives smallest $\ell_2$ error $\textup{argmin}_{C \geq 0} \frac{1}{2}\sum_{\gamma_j \in B_i}(\gamma_j - C)^2$) is : $\max \left\{ \frac{1}{|B_i|}\sum_{\gamma_j \in B_i} \gamma_j, 0 \right\}$. Meanwhile, it's feasible to set 
$$
(\Pi_{\mathcal{S}}(\mathbf{\gamma}))_i = \max \left\{ \frac{1}{|B|}\sum_{\gamma_j \in B} \gamma_j, 0 \right\}
$$
since we have that $\max \left\{ \frac{1}{|B_i|}\sum_{\gamma_j \in B_i} \gamma_j, 0 \right\} \geq \max \left\{ \frac{1}{|B_{i+1}|}\sum_{\gamma_j \in B_{i+1}} \gamma_j, 0 \right\}$ by definition of blocks. This wraps up the proof.

\end{proof}

\subsection{Proof of Theorem \ref{pgd}}
We next prove the validity of Algorithm \ref{alg: projection}.

\begin{proof}
Suppose $\mathbf{\gamma}$ has segmentation blocks $B_1, \dots, B_q$, we firstly prove that $(\Lambda_{\mathcal{S}}(\mathbf{\gamma}))_i = (\Pi_{\mathcal{S}}(\mathbf{\gamma}))_{i}$ for $i \leq |B_1|$. We let $\gamma_j(t)$ denote the value of $\gamma_j$ at the moment $i$ was assigned from $t$ to $t+1$ in Algorithm \ref{alg: projection} (i.e. the time when first $t$ iterations are finished). We also let $\gamma_j(0)$ denote the initial value of $\gamma_j$ in the input. Then clearly $(\Lambda_{\mathcal{S}}(\mathbf{\gamma}))_j = \max\{\gamma_j(p),0\}$. During the value-averaging step, the algorithm is constantly transporting values from elements with larger index to those with smaller. Hence it's straightforward to see that
\begin{equation} \label{thm2:1}
    \sum_{j=1}^{J}\gamma_j(t) \geq \sum_{j=1}^{J}\gamma_j(t-1)
\end{equation}
for arbitrary $J, t \in \{1, \dots, p\}$. First assume $\gamma_1(p) = \cdots = \gamma_{\widetilde{L}_1}(p) > \gamma_{\widetilde{L}_1+1}(p)$. Since Algorithm \ref{alg: projection} only involves averaging values among subsequences, we have that $\sum_{j=1}^p \gamma_j(p) = \sum_{j=1}^p \gamma_j$. Moreover since $\gamma_{\widetilde{L}_1}(p) > \gamma_{\widetilde{L}_1+1}(p)$, there's no value-averaging steps between any one of the first $\widetilde{L}_1$ elements and one of the rest elements. This implies \begin{equation}\label{thm2:2}
\sum_{j=1}^{\widetilde{L}_1}\gamma_j(p) = \sum_{j=1}^{\widetilde{L}_1} \gamma_j
\end{equation}
By definition of blocks, we know that $\exists k$ such that $\frac{1}{k}\sum_{i=1}^{k} \gamma_{\widetilde{L}_1+i} \geq \frac{1}{\widetilde{L}_1}\sum_{i=1}^{\widetilde{L}_1} \gamma_{i} = \gamma_1(p)$. By (\ref{thm2:1}) we have that 
$$
\frac{1}{k}\sum_{i=1}^{k} \gamma_{\widetilde{L}_1+i} \leq \frac{1}{k}\sum_{i=1}^{k} \gamma_{\widetilde{L}_1+i}(p) \leq \gamma_{\widetilde{L}_1+1}(p)
$$
Together with above, this implies that $\gamma_{1}(p) \leq \gamma_{\widetilde{L}_1+1}(p)$, which contradicts to the assumption. Hence we have that $\widetilde{L}_1 \geq L_1$.

On the other hand, if $\widetilde{L}_1 > L_1$, then at the moment $i$ is assigned to be $\widetilde{L}_1 + 1$ in the algorithm (i.e. the time when first $\widetilde{L}_1$ iterations are finished), we must have that 
$$
\frac{\sum_{j=1}^{\widetilde{L}_1}\gamma_j(\widetilde{L}_1-1)}{\widetilde{L}_1} \geq \frac{\sum_{j=1}^{{L}_1}\gamma_j(\widetilde{L}_1-1)}{{L}_1}
$$
This implies that 
\begin{equation} \label{thm2:3}
    \frac{\sum_{j=L_1+1}^{\widetilde{L}_1}\gamma_j(\widetilde{L}_1-1)}{\widetilde{L}_1 - L_1} \geq \frac{\sum_{j=1}^{{L}_1}\gamma_j(\widetilde{L}_1-1)}{{L}_1}
\end{equation}
By (\ref{thm2:1}) we have
\begin{equation} \label{thm2:4}
\frac{\sum_{j=1}^{{L}_1}\gamma_j}{{L}_1} \leq \frac{\sum_{j=1}^{{L}_1}\gamma_j(\widetilde{L}_1-1)}{{L}_1}
\end{equation}
Meanwhile at $t = \widetilde{L}_1-1$, the sum of first $L_1$ terms is the same as that in $\mathbf{\gamma}$. This implies
\begin{equation} \label{thm2:5}
\begin{aligned}
    \sum_{j = L_1+1}^{\widetilde{L}_1}\gamma_j(\widetilde{L}_1-1) & = \sum_{j=1}^{L_1}\gamma_j + \sum_{j=L_1+1}^{\widetilde{L}_1}\gamma_j - \sum_{j = 1}^{{L}_1}\gamma_j(\widetilde{L}_1-1) \\
    & \leq \sum_{j=L_1+1}^{\widetilde{L}_1}\gamma_j
\end{aligned}
\end{equation}
where the last inequality is given by (\ref{thm2:1}). Combining (\ref{thm2:3}), (\ref{thm2:4}) and ((\ref{thm2:5})) yields
$$
\frac{\sum_{j=L_1+1}^{\widetilde{L}_1}\gamma_j}{\widetilde{L}_1 - L_1} \geq \frac{\sum_{j=1}^{{L}_1}\gamma_j}{{L}_1}
$$
This contradicts to definition of $L_1$ in (\ref{L}). Hence we have that $\widetilde{L}_1 = L_1$. This means $\gamma_1(p) = \cdots = \gamma_{{L}_1}(p) > \gamma_{{L}_1+1}(p)$. Recall that $(\Lambda_{\mathcal{S}}(\mathbf{\gamma}))_j = \max\{\gamma_j(p),0\}$, this together with (\ref{thm2:2}) yields
$$
\begin{aligned}
(\Pi_{\mathcal{S}}(\mathbf{\gamma}))_1 & = \max \left\{ \frac{1}{|B_1|}\sum_{j = 1}^{L_1} \gamma_j, 0 \right\} = (\Lambda_{\mathcal{S}}(\mathbf{\gamma}))_1 \\
& = \cdots = (\Lambda_{\mathcal{S}}(\mathbf{\gamma}))_{{L}_1} > (\Lambda_{\mathcal{S}}(\mathbf{\gamma}))_{{L}_1+1}
\end{aligned}
$$
Now we have prove that $(\Pi_{\mathcal{S}}(\mathbf{\gamma}))_i = (\Lambda_{\mathcal{S}}(\mathbf{\gamma}))_i$ for $i \leq |B_1|$ and that there is no interaction between element in $B_1$ and that outside $B_1$. This implies that the existence of $B_1$ does \textit{not} affect the rest of output values $(\Lambda_{\mathcal{S}}(\mathbf{\gamma}))_{i > |B_1|}$. Hence we can ignore $B_1$ and repeat exactly the same procedure to prove that $(\Pi_{\mathcal{S}}(\mathbf{\gamma}))_i = (\Lambda_{\mathcal{S}}(\mathbf{\gamma}))_i$ when $|B_1|+1 \leq i \leq |B_2|$ and that there is no interactions between element in $B_2$ and that outside $B_2$. Iteratively we can prove $\Pi_{\mathcal{S}}(\mathbf{\gamma}) = \Lambda_{\mathcal{S}}(\mathbf{\gamma})$

\end{proof}

\end{document}